\documentclass{article}

\usepackage{microtype}
\usepackage{graphicx}
\usepackage{subfigure}
\usepackage{booktabs} %

\usepackage{hyperref}

\usepackage{amsfonts}
\usepackage{amsmath}
\usepackage{amsthm}
\usepackage{amssymb}
\usepackage{dsfont}
\usepackage{pifont}

\usepackage{empheq}

\usepackage{xcolor}
\usepackage{color}
\usepackage{graphicx}

\usepackage{verbatim}
\usepackage{xspace} %

\usepackage{enumerate}
\usepackage{enumitem}

\definecolor{lightblue}{HTML}{BBEEFF}

\newcommand*{\colorboxed}{}
\def\colorboxed#1#{%
  \colorboxedAux{#1}%
}
\newcommand*{\colorboxedAux}[3]{%
  \begingroup
    \colorlet{cb@saved}{.}%
    \color#1{#2}%
    \boxed{%
      \color{cb@saved}%
      #3%
    }%
  \endgroup
}

\definecolor{boxcolour}{HTML}{ffe9f4}

\providecommand{\abs}[1]{\left\lvert#1\right\rvert}
\providecommand{\norm}[1]{\left\lVert#1\right\rVert}

\providecommand{\R}{\mathbb{R}} %

\DeclareMathOperator{\E}{{\mathbb E}}

\providecommand{\0}{\mathbf{0}}

\renewcommand{\aa}{\mathbf{a}}
\providecommand{\bb}{\mathbf{b}}

\providecommand{\uu}{\mathbf{u}}

\providecommand{\xx}{\mathbf{x}}
\providecommand{\yy}{\mathbf{y}}

\providecommand{\mA}{\mathbf{A}}

\providecommand{\cO}{\mathcal{O}}

\providecommand{\cS}{\mathcal{S}}

\usepackage{bm}

\usepackage{url}

\renewcommand{\epsilon}{\varepsilon}

\def\beq{\begin{equation}}
\def\eeq{\end{equation}}
\def\ba{\begin{array}}
\def\ea{\end{array}}
\def\la{\langle}
\def\ra{\rangle}
\definecolor{mydarkgreen}{RGB}{39,130,67}

\definecolor{mydarkred}{RGB}{192,47,25}

\DeclareMathOperator{\sigmasgd}{\sigma^2_{\operatorname{SGD}}}

\DeclareMathOperator{\sigman}{\sigma^2_{n}}

\DeclareMathOperator{\sigmakt}{\sigma^2_{k,\tau}}

\DeclareMathOperator{\sigmat}{\sigma^2_{\tau}}

\DeclareMathOperator{\sigmaepoch}{\sigma^2_{\operatorname{EPOCH}}}

\DeclareMathOperator{\sigmaone}{\sigma^2_{\operatorname{ONE}}}

\DeclareMathOperator{\epsilontt}{\epsilon^{3/2}}

\DeclareMathOperator\mymod{{{\rm \, mod \,}}}

\usepackage[accepted]{my_icml2024}

\usepackage{amsmath}
\usepackage{amssymb}
\usepackage{mathtools}
\usepackage{amsthm}
\usepackage{nicefrac}       %

\usepackage[capitalize,noabbrev]{cleveref}

\theoremstyle{plain}
\newtheorem{theorem}{Theorem}[section]

\newtheorem{lemma}[theorem]{Lemma}

\theoremstyle{definition}
\newtheorem{example}[theorem]{Example}
\newtheorem{definition}[theorem]{Definition}
\newtheorem{assumption}[theorem]{Assumption}
\theoremstyle{remark}

\usepackage[textsize=tiny]{todonotes}

\icmltitlerunning{On Convergence of Incremental Gradient for Non-Convex Smooth Functions}

\begin{document}

\twocolumn[
\icmltitle{On Convergence of Incremental Gradient for Non-Convex Smooth Functions}

\icmlsetsymbol{equal}{*}

\begin{icmlauthorlist}
\icmlauthor{Anastasia Koloskova}{epfl}
\icmlauthor{Nikita Doikov}{epfl}
\icmlauthor{Sebastian U. Stich}{cispa}
\icmlauthor{Martin Jaggi}{epfl}
\end{icmlauthorlist}

\icmlaffiliation{epfl}{Machine Learning and Optimization Laboratory (MLO), EPFL, Lausanne, Switzerland}
\icmlaffiliation{cispa}{CISPA Helmholtz Center for Information Security, Saarbrücken, Germany}

\icmlcorrespondingauthor{Anastasia Koloskova}{anastasia.koloskova@epfl.ch}
\icmlcorrespondingauthor{Nikita Doikov}{nikita.doikov@epfl.ch}
\icmlcorrespondingauthor{Sebastian U. Stich}{stich@cispa.de}
\icmlcorrespondingauthor{Martin Jaggi}{martin.jaggi@epfl.ch}

\icmlkeywords{Machine Learning, stochastic optimization, SGD without replacement, non-convex smooth functions}

\vskip 0.3in
]

\printAffiliationsAndNotice{}  %

\begin{abstract}
    
    In machine learning and neural network optimization, algorithms like incremental gradient, and shuffle SGD are popular due to minimizing the number of cache misses and good practical convergence behavior. However, their optimization properties in theory, especially for non-convex smooth functions, remain incompletely explored. 
    
    This paper delves into the convergence properties of SGD algorithms with arbitrary data ordering, within a broad framework for non-convex smooth functions. Our findings show enhanced convergence guarantees for incremental gradient and single shuffle SGD. Particularly if $n$ is the training set size, we improve $n$ times the optimization term of convergence guarantee to reach accuracy $\epsilon$ from $\cO \left( \nicefrac{  {{n} }  }{\epsilon} \right)$ to $\cO \left( \nicefrac{ {{1} } }{\epsilon} \right)$.

\end{abstract}

\section{Introduction}
In this paper we study the problem of minimizing the finite-sum objective:
\begin{align}\label{eq:problem}
    \textstyle\min_{\xx \in \R^d} \Big[f(\xx) := 
    \frac{1}{n}\sum_{i = 1}^n f_i(\xx)\Big].
\end{align}
We denote by $n$ the number of functions $f_i \colon \R^d \to \R$. Every function $f_i$
is assumed to be smooth and can be non-convex.
This optimization problem arises in many practical applications. For example, in the training of machine learning models with $n$ being the training set size and each $f_i(\xx)$ being the loss of the model on the $i$-th datapoint,
while $\xx$ is the vector of the model parameters.

A common method for solving optimization problems of the form~\eqref{eq:problem}
is the Stochastic Gradient Descent (SGD) algorithm
and its various modifications \cite{lan2020first}.
During each iteration $t \geq 0$ of the method, an index
$i_t \in [n] := \{1, \ldots, n\}$ (or a subset of indices) 
is selected upon which a gradient step is performed: 
\beq \label{eq:algo}
\ba{rcl}
\xx_{t + 1} =  \xx_t - \gamma \nabla f_{i_{t}}(\xx_t),
\ea
\eeq
where $\gamma > 0$ is a step size.
Defining how the index $i_t$ is selected in \eqref{eq:algo} is crucial for the performance of this algorithm.
The standard theoretical analysis often assumes 
that the index $i_t \sim [n]$ is chosen uniformly at random (a variant we refer to as SGD in this paper). However, in practice, a specific
order of selecting the gradients can be \textit{predefined} by a configuration of the system. For example, many standard frameworks use dataloaders that process data in epochs. Within each epoch, we access all indices
in a permutation.
The permutation can vary across epochs,
or fixed at once in the beginning. The latter is one of the most popular approaches in practice due to simplicity of its implementation
and high efficacy of memory usage, minimizing
the number of cache misses.
From the optimization perspective, this means that after performing the gradient step~\eqref{eq:algo} for a
certain index $i_t$, the next time we access the function $f_{i_t}$ again
is \textit{exactly after $n$ iterations}.
Despite being widely approved by the technical and engineering needs,
these variations of SGD present one of the most challenging scenario for the analysis, because of mutual dependencies between consecutive steps.\looseness=-1

\begin{figure*}[h!]
    \vspace{-0.92em}
    \centering
    \subfigure[\centering \scriptsize Convergence of the gradient norm with epochs $(d = 100, n = 50)$]{\label{fig:quadr_converg}\includegraphics[width=0.63\linewidth]{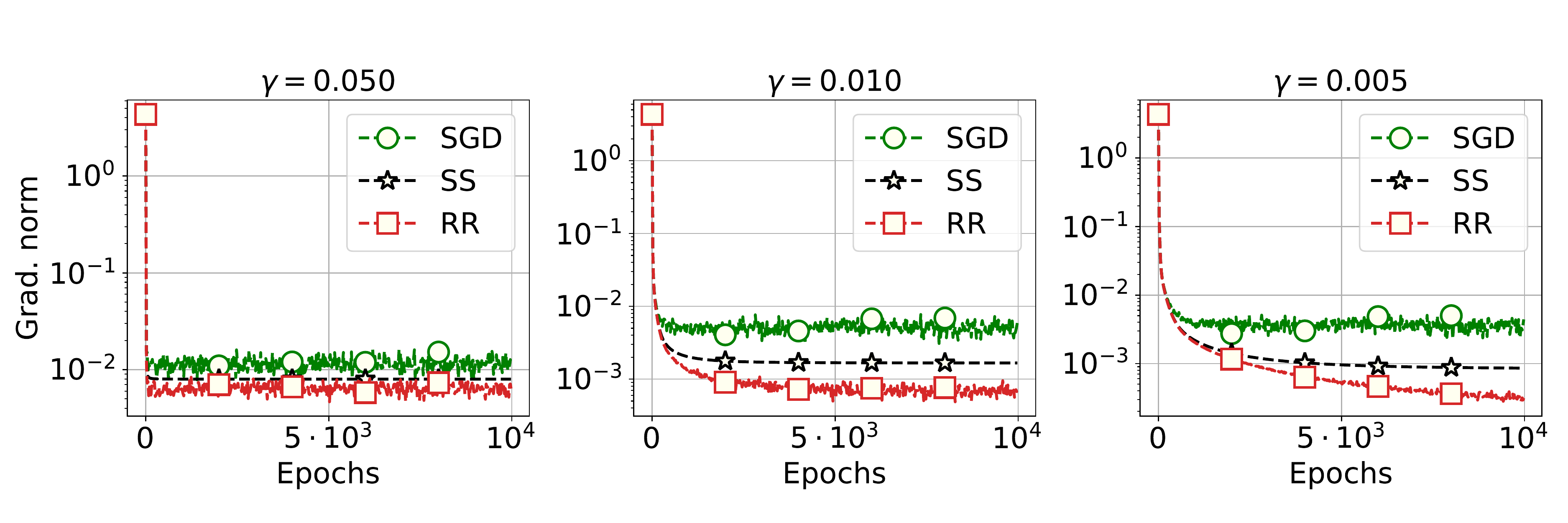}}%
    \qquad
    \subfigure[\centering  \scriptsize Number of iterations to reach $\varepsilon$ accuracy (the lower is the better).]{\label{fig:quadr_n}{\includegraphics[width=0.25\linewidth]{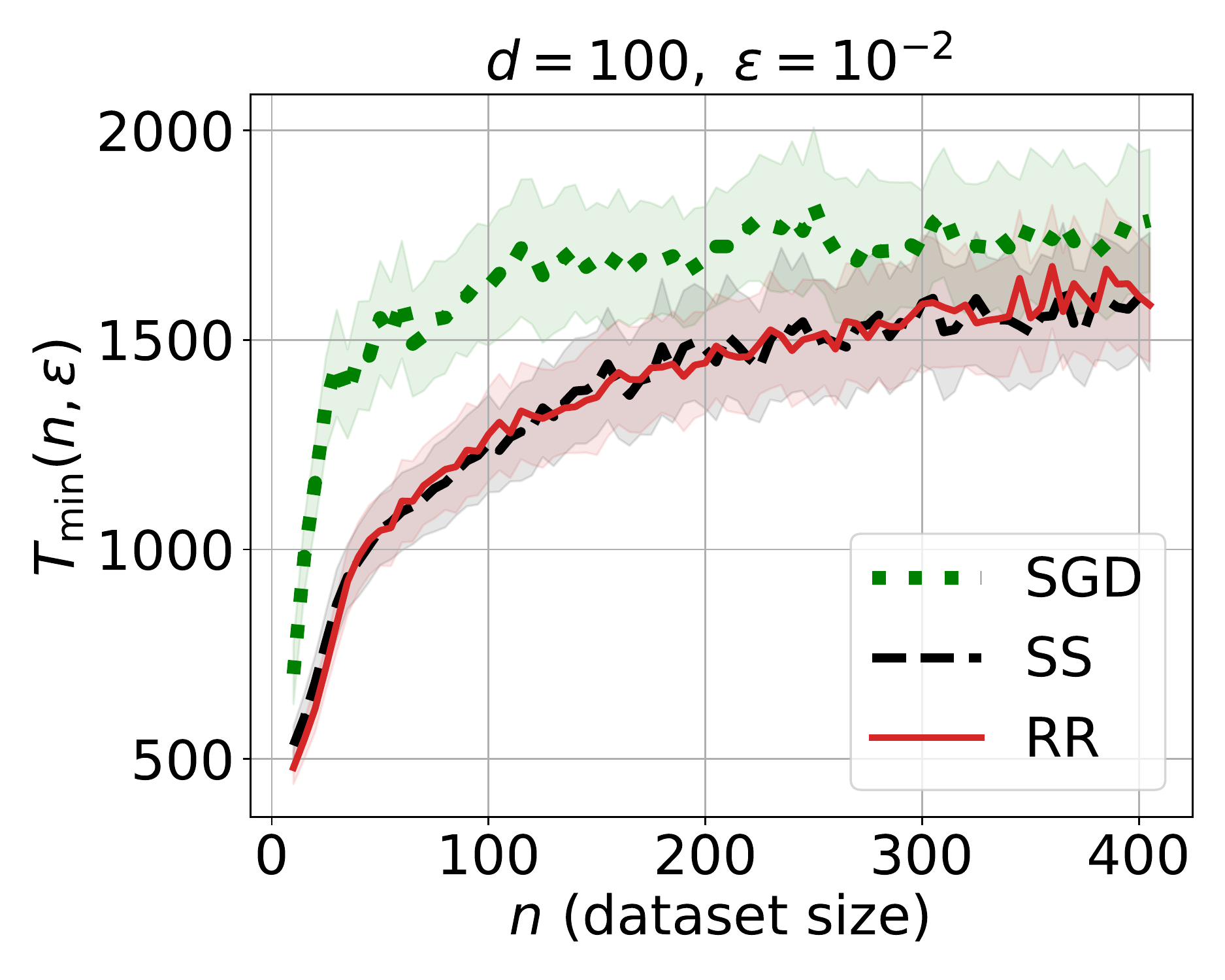} }}%
    \vspace{-1em}
    \caption{\small Minimizing stochastic quadratic function for different strategies of sampling the gradients.
    Random Reshuffling (RR) and Single Shuffling (SS) work better than SGD.}%
    \label{fig:Quadratic}\vspace{-2mm}
\end{figure*}

In this paper, our goal is to understand the convergence
properties of algorithm~\eqref{eq:algo}
with a predefined selection of indices shared within one epoch.
An important particular case is called
the \emph{incremental gradient descent} (IG) algorithm.
In recent decades, it has received significant attention 
in the optimization literature
\cite{bertsekas2011incremental}.
The basic version of this method simply substitutes
\beq \label{eq:intro_ic}
\ba{rcl}
i_{t} & := & t \mymod n + 1 
\ea
\eeq
into algorithm \eqref{eq:algo}.
Note that when the initial order of functions
in \eqref{eq:problem}
is randomly permuted once in the beginning,
this selection rule is equivalent 
to the so-called Single Shuffle (SS)
method, as opposed to Random Reshuffle (RR),
which permutes the functions after every epoch.
Even when putting aside cache efficiency in practical implementations, both SS and RR show good practical behavior often surpassing classic SGD with uniform sampling (see e.g. Fig.~\ref{fig:Quadratic}). This makes them popular in the machine learning and optimization community \citep{Mishchenko2020RandomRS,lu2022a:shuffle} and very often these are the methods of choice that are implemented in practical systems.
Another important example of a significant benefit of the incremental algorithms is second-order optimization. Thus, it was shown in \cite{rodomanov2016superlinearly} that the cyclic rule \eqref{eq:intro_ic} for updating a second-order model in the Newton methods is essential to preserve the local superlinear convergence.

Despite their popularity and wide practical approval, 
the convergence rates for 
the methods with
arbitrary data orders
and, specifically, for the incremental gradients~\eqref{eq:intro_ic} are not well understood. 
Up to our knowledge, the best results for IG show a dependency on the dataset size $n$ of the form $\Omega\bigl(\frac{n}{\epsilon}\bigr)$, i.e.\ a linear dependence on $n$~\citep{Mishchenko2020RandomRS,lu2022a:shuffle}.
This means that $n$-iterations of IG give comparable progress to one step of the (full-batch) gradient method. This is in contrast to convergence rates of SGD, for which the optimization term is $\mathcal{O}\bigl(\frac{1}{\epsilon}\bigr)$ and does not depend on $n$. This is also in contrast to empirical observations. \looseness=-1

In this work, we provide a novel analysis for a general family
of algorithms of the form~\eqref{eq:algo},
with an emphasis on the incremental gradient methods
and establish new convergence rates
for non-convex optimization that are significantly better
than the previously known ones (see Table~\ref{tab:rates}),
thus advancing in a long-standing problem in optimization theory.\looseness=-1

In order to achieve such an improvement, we develop a new proof technique, drawing inspiration from previous work on SGD with linearly correlated noise, such as occurring in optimization with differential privacy \cite{koloskova23:correlated_noise}. In particular, we divide all iterations $t = 0, \dots, T$ of the method into smaller chunks of size $\tau = \Theta\left(\nicefrac{1}{L \gamma}\right)$ with $L$ being the smoothness constant and~$\gamma$ the step size, and analyze each of the chunks separately. %
This is different from the previous works as they usually consider the correlation periods of a fixed size $n$, regardless of how large $n$ is.

We can summarize our contributions as follows:
\begin{itemize}[leftmargin=12pt,itemsep=1pt,topsep=1pt]
    \item We derive convergence rates for SGD with \emph{arbitrary} data orderings for non-convex smooth functions, that cover a wide range of algorithms, including, but not restricted to Random Reshuffling (RR), Single Shuffle (SS) and Incremental Gradient (IG).
    
    \item Our convergence rates strictly improve all the rates in prior works for the non-convex smooth functions in the case of the Incremented Gradient
    and Single Shuffle (see Table~\ref{tab:rates}).\looseness=-1
    
    \item Our theoretical analysis technique is novel and of independent interest in comparing other schemes.
\end{itemize}
\begin{table*}[t]
\centering
\resizebox{0.7\linewidth}{!}{
\begin{tabular}{c|c|c}
\toprule
Algorithm & Prior works & This work \\ \midrule
Incremental Gradient / \\ Single Shuffle
& $\, L F_0 \cdot \cO \Bigl( \frac{  {\color{mydarkred} \boldsymbol{n} }  }{\epsilon} + \frac{n\sigma_{\operatorname{SGD}}}{\epsilontt} \Bigr)$ 
&  $L F_0 \cdot \cO \Bigl( \frac{ {\color{mydarkgreen} \boldsymbol{1} } }{\epsilon} +
{\color{mydarkgreen} \boldsymbol{\min}}\Bigl\{ \frac{n \sigma_{\operatorname{SGD}} }{\epsilontt}, 
{\color{mydarkgreen} \boldsymbol{\frac{n \sigma^2_{\operatorname{SGD}}}{\epsilon^{2}} }} \Bigr\}
\Bigr)$\vspace{0.5mm}
\vspace{0.5mm}\\
\bottomrule
\end{tabular}
}  %
\vspace{1mm}
\caption{Prior best known \citep{Mohtashami2022data_order, lu2022a:shuffle, Mishchenko2020RandomRS} complexity (number of iterations to achieve accuracy $\epsilon$) for the special case of
Incremental Gradient and Single Shuffle methods covered in our framework, compared to the improved rate derived in our work. $\cO( \cdot )$ hides an absolute numerical constant. 
Note that the standard complexity of SGD is $LF_0 \cdot \cO( \frac{1}{\epsilon} + \frac{\sigma^2_{\operatorname{SGD}}}{\epsilon^2}  )$, which is also covered by our
analysis as a special case \eqref{eq:sgd_rate}.
However, the practical behaviour of Single Shuffle is usually better than that of SGD 
(see Fig.~\ref{fig:Quadratic}
and Section~\ref{SectionExperiments} with our numerical experiments).
}\vspace{-4mm}
\label{tab:rates}
\end{table*}

\section{Related work}
In this section, we review some of the most well-established practical and theoretical results on incremental gradient, shuffle SGD variants and SGD with other data orders that were known prior to our work.

\textbf{Practical observations.}
Preceding any theoretical findings,
it was initially empirically discovered that the behaviour
of  stochastic methods can vary significantly depending
on the order of stochastic gradients.
It has been widely reported in the literature that
random shuffling or cycling through the data over a fixed shuffle (permutation)
yields better convergence  
than sampling the datapoints uniformly at random.
For instance, \cite{nedic2001incremental}
discovered this in the context of subgradient methods, 
while
\cite{Bottou2009CuriouslyFC} and \cite{shalev2011pegasos} %
were among the first to discuss this observation in the context of machine learning problems.
They posed an open problem to justify this theoretically.
\cite{Recht11:jellyfish} showed
that for the matrix completion problem, random shuffling can be several orders of magnitude faster.
Similarly, in the context of neural network training, 
it was recommended by \cite{Bengio2012PracticalRF} to shuffle the dataset once and then
use a fixed order of the gradients.
Currently, random reshuffling is standard practice in training deep learning models
\citep{paszke2019pytorch,sun2020optimization}. It is also used routinely for training large language models
\citep{Chowdhery2022PaLM, touvron2023llama}, with a small number of epochs
over the randomly shuffled training data.
Shuffling methods were predated by the incremental gradient methods that pass over the data in a given order~\citep{kohonen1974cyclic,luo1991convergence,grippo1994class,Bertsekas15:incrementalGD}.
Studies on the effects of data ordering on neural network training can be traced back at least to the 1960s~\cite{widrow1960adaptive}. \looseness=-1

\textbf{Theoretical guarantees.} 
There were many attempts and breakthroughs to theoretically explain
the  good empirical behaviour of the 
gradient methods with reshuffling for different problem classes.
Convergence rates of the incremental gradient methods with random reshuffling
for \textit{convex} optimization
are dated back to \citep{nedic2001convergence}
(see also \citep{bertsekas2011incremental}).
Over the recent years, significant attention 
has been dedicated to the \textit{strongly convex} case.
A series of works could establish that RR can converge faster than SGD.
Among the first analyses is \cite{recht2012toward} that focuses on quadratic least squares problem.
Follow-up work also focuses on quadratics or 
relaxes the assumption by imposing second-order smoothness assumptions.
\citet{Gurbuzbalaban15:shuffling,Haochen19:random_shuffle_first,Safran20:random_shuffle_lower_bound} provided the lower bounds for RR and SS for strongly convex functions. Subsequently, 
\citep{Rajput20:shuffle_sgd_quadratics} showed how to match the lower bound for RR when the function~$f$ is quadratic and after a large enough number of epochs. For linear regression \citet{Yun2021CanSS} show that SS is better than RR and better than SGD. 
The second-order smoothness assumption could be dropped in \citep{Jain2019shuffle,Safran21:shuffle}  %
and \cite{Mishchenko2020RandomRS}. The latter makes the observation that introducing a specific notion of variance that takes the random permutation into account could facilitate the analysis in the convex and strongly convex setting.
The results were extended to the non-convex setting under the PL condition \citep{Ahn20:shuffle,Nguyen20:unified_shuffling} 
and for the general \emph{non-convex smooth} setting in \citep{lu2022a:shuffle, Mishchenko2020RandomRS, Mohtashami2022data_order}.
Tighter lower complexity bounds for strongly convex functions and functions satisfying the PL condition were established 
recently in \citep{cha2023tighter}.
In the last years, a significant attention was paid to the distributed and federated stochastic methods with random reshuffling
(see \citep{yun2021minibatch,sadiev2022federated,malinovsky2022server,mishchenko2022proximal,cho2023convergence} and references therein), which are important  
for training large-scale and decentralized models.

Moreover, all of the convergence rates for non-convex functions require the number of epochs to be unreasonably large for random reshuffling to be better than SGD.

\textbf{Arbitrary orderings.}
\citep{lu2022a:shuffle, Mohtashami2022data_order} analysed arbitrary data orderings in SGD (not restricted to permutations), including shuffle SGD and proposed an algorithm of greedy data selection to select data orders that would lead to the fastest convergence, based on their analysis. 
\cite{lu2022grab} proposed a more practical (than greedy order) data order selection that is faster than random shuffle.\looseness=-1

There exist many more applications where different data orderings appear naturally, such as in Markovian SGD \citep{Johansson09:markov,Duchi11:markovmirrordescent,even2023markov} or certain federated learning scenarios~\cite{eichner2019semi}.
\citet{yun2022minibatch} analyse local SGD for distributed learning when every node applies random reshuffling.
Using a provably faster permutation-based example ordering in distributed training was recently studied in \citep{feder2023cd}.\looseness=-1

\section{The Algorithm}

Recall that we study the finite-sum minimization problem 
\begin{align}
    \textstyle\min_{\xx \in \R^d} \Big[f(\xx) := \frac{1}{n}\sum_{i = 1}^n f_i(\xx)\Big] \tag{\ref{eq:problem}}
\end{align}
where $n$ denotes the number of functions. %

While we are mainly interested in analysing convergence properties of the incremental gradient (IG) algorithm, we will study it under the more general framework that allows for \emph{arbitrary} data ordering:
\begin{align} %
\xx_{t + 1} = \xx_t - \gamma \nabla f_{i_t}(\xx_t) \tag{\ref{eq:algo}}
\end{align}
where $t = 0, \dots, T$ and $i_t \in [n] := \{1,\dots,n\}$ denotes the index of the datapoint chosen at iteration~$t$.

In our work we allow for any possible strategies of choosing datapoints orders, i.e.\ sequences of indices $(i_t)_{0 \leq t \leq T} := (i_0, \dots, i_t, \dots, i_T)$---deterministic or random.

The results that we present in this work can be straightforwardly extended to the {mini-batch setting}, 
where in each iteration a mini-batch gradient is computed on a subset of the data. 
In this case, we can denote by $I_t$ a set of indices
taken in iteration $t$ and work with $f_{I_t}(\xx) := \frac{1}{\abs{I_t}} \sum_{j \in I_t} f_j(\xx)$.
However, for simplicity of the presentation, we will mainly focus on the iterations of the form \eqref{eq:algo},
and $i_t$ being just one index.\looseness=-1

Let us list some examples of the popular algorithms covered in our framework \eqref{eq:algo}:
\begin{example}[SGD]\label{ex:sgd}
The classic SGD algorithm~\cite{Robbins51:sgd} samples in each iteration~$t$ the next data sample uniformly at random from $[n]$, i.e.\ $i_t \sim [n]$. %
This means that in one epoch ($n$ iterations), some samples may be selected more than once and others may be missing (sampling with replacement).
Most theoretic analyses of stochastic methods (see, e.g.\ the book by \cite{lan2020first}) are based on this type of sampling, since independently choosing a data point in each iteration significantly facilitates the proofs.
\end{example}

\begin{example}[Incremental Gradient (IG)]\label{ex:ig}
The incremental gradient method passes over the data samples in cycles (which we will also refer to as epochs), i.e.\ $i_t = t \mymod n + 1$. This method has been among the earliest used for training neural networks \citep{kohonen1974cyclic,luo1991convergence,bertsekas2011incremental} but comes with unfavorable worst-case convergence guarantees~\citep{gurbuzbalaban2019convergence}.
\end{example}

\begin{example}[Single Shuffle (SS)]\label{ex:SS}
Single shuffle SGD also passes over the data sample in cycles, but instead of using the predefined order, the data order is determined by a randomly chosen permutation $\pi$ of the index set $[n]$, and we set $i_t = \pi(t \mymod n + 1)$. %
This variant is also sometimes used in practical implementations \citep{Bengio2012PracticalRF}.
\end{example}

\begin{example}[Random Reshuffling (RR)]\label{ex:RR}
SGD with random reshuffling (RR) \cite{nedic2001convergence} is a further popular variant that passes in cycles of length $n$ over the data. The main difference compared to SS is that in RR a \textit{new} permutation $\pi_k$ is drawn uniformly at random at the beginning of each epoch $k$, 
and then $i_t = \pi_k(t\mymod n+1)$,
which is also called \textit{sampling without replacement}. This variant is frequently a default option in training modern neural networks \citep{Goyal17:imagenet1h,paszke2019pytorch}.
\end{example}

The following example is less conventional in the literature, it serves to highlight that our framework also allows orderings that completely disregard parts of the data, leading to biased training.\looseness=-1

\begin{example}[Single Function]\label{ex:singlefunc}
    Since our algorithm \eqref{eq:algo} allows for any orders of the data, it is also allowed to always sample the same (for example the first) function, i.e. $i_t \equiv 1~\forall t$. In this case, we cannot expect the convergence to the exact optima of $f$ in \eqref{eq:problem}, but to the minimizer of $f_1$. We will see that in this case, our theory can quantify the neighbourhood size to which the algorithm converges.
    While this example is only illustrative, it serves as a proxy to a more realistic scenario when some functions might be missing due to e.g.\ being distributed over the nodes/devices that are offline. \looseness=-1%
\end{example}
Any other deterministic or randomized orders are also possible, such as Markovian SGD \citep{Johansson09:markov, even2023markov}.\looseness=-1%

Our analysis will cover arbitrary orders and include all of the examples from Examples~\ref{ex:sgd} to \ref{ex:singlefunc}. However, it is important to note that not all examples will demonstrate improved convergence under our analysis. The primary focus of our study will be on the Incremental Gradient and Single Shuffle SGD algorithms (Examples~\ref{ex:ig}, \ref{ex:SS}), where we notably improve existing convergence guarantees.

\section{Assumptions}

For our theoretical analysis, we will use the following standard assumption on the smoothness of functions (see, e.g.~\cite{nesterov2018lectures}).\looseness=-1
\begin{assumption}\label{as:smooth}
Each of $f_i$ is $L$-smooth, i.e.
$\norm{\nabla f_i(\xx) - \nabla f_i(\yy)} \leq L \norm{\xx - \yy}\,, ~~\forall \xx,\yy \in \R^d\,.$
\end{assumption}
\subsection{Quantifying the data orders}
Different data orders $(i_0, \dots, i_T)$ in Algorithm~\eqref{eq:algo} lead to the different practical performance and the different convergence rates. Consider the two examples: (i) $i_t$ is chosen uniformly at random from $[n]$, corresponding to the SGD Example~\ref{ex:sgd} and (ii) $i_t$ is always chosen to be the first function $i_t\equiv 1$, as in Example~\ref{ex:singlefunc}. Clearly while SGD can effectively minimize the original objective function $f$ defined in \eqref{eq:problem}, the second algorithm converges to the minimum of $f_1$, and might not converge to the minima of $f$ in general, if $f_1 \neq f$.\looseness=-1

In order to provide a unified analysis that captures the convergence of arbitrary data ordering in Algorithm~\eqref{eq:algo}, we need to introduce a quantity that measures how does the datapoint order $(i_0, \dots, i_T)$ affect the convergence rate. In this work, we propose to use the following quantity
\begin{definition}[Sequence correlation]
\label{def:main_variance}
Let $(i_0,\dots,i_T)$ denote a (possibly random) sequence of indices $i_t \in [n], t=0,\dots, T$. 
For a given  $\tau \geq 1$ (that we call \emph{effective correlation time}), we divide the full sequence  into  $\big\lfloor \frac{T}{\tau}\big\rfloor$ consecutive chunks of size $\tau$. We call these chunks as \emph{correlated periods}. We further define for every period $k=0, \dots, \big\lfloor \frac{T}{\tau}\big\rfloor$, the sequence variance as
    \begin{align}
        &\sup_{\xx\in\R^d}\max_{\substack{k = 0, \dots, \lfloor \frac{T}{\tau}\rfloor \\ j=0, \dots, \tau-1}} \E\left[\phi_{k\tau + j}(\xx) \Big| ~ i_0, \dots, i_{k \tau - 1}\right]\leq \sigmat
    \end{align}
    where 
    \begin{align}
        \phi_{k\tau + j}(\xx) = \Bigg\|\sum_{t = k \tau}^{\min\{k \tau + j, T\}} \left(\nabla f_{i_t}(\xx) - \nabla f(\xx)\right) \Bigg\|^2\nonumber
    \end{align}
    where the expectation is taken over the choice of the random sequence $(i_t)_{k\tau \leq t \leq T}$ conditioning on the past sequence order $(i_t)_{t < k\tau}$.
    We note that $\sigma_{\tau}^2$ depends on the distribution from which the sequence $(i_t)_{0 \leq t \leq T}$ is drawn, 
    but we omit this dependence in the text when it is clear from the context.
\end{definition}
For the special case of the correlated periods of size one, i.e. $\tau = 1$ and indices sampled uniformly at random, $i_t \sim [n]$, this definition recovers the standard SGD variance as used in~\eqref{eq:sigmasgd} below.
When considering $\tau > 1$, this measure can capture the joint effect of sequence order and the gradients of the individual functions. This measure shows how close the average of prefix of individual gradients $\nabla f_{i_t}(\xx)$ stays to the full gradient $\nabla f(\xx)$ during the period of $\tau$ consecutive steps.\looseness=-1

\subsection{Comparison to the prior quantities}
We will now explore how the sequence correlation quantity we introduced earlier, as defined in Definition~\ref{def:main_variance}, relates to similar concepts in existing literature. Additionally, we will discuss the motivations behind proposing this new sequence correlation approach.\looseness=-1

We first start with comparing our sequence correlation with classic variance assumptions that appear in the analysis of stochastic gradient methods \citep{Bubeck15:convexopt,lan2020first}.
\paragraph{Classic bounded variance assumption.}

One of the most common assumptions used in analysing SGD-type algorithms \citep{Lan12:opt, Dekel12:minibatch} is bounded variance of the gradients. They assume that $\sigmasgd$ is finite, where $\sigmasgd$ is defined as\vspace{-1mm}
\begin{align}
   \sigmasgd := \textstyle\sup_{\xx \in \R^d} \E_i \norm{\nabla f_i(\xx) - \nabla f(\xx)}^2\,,
   \label{eq:sigmasgd}
\end{align}
and the expectation is taken over the uniform choice of the index $i \sim [n]$ (note that mini-batches are also allowed, effectively dividing this variance quantity by the batch size).

While this measure $\sigmasgd$ is good in quantifying the behaviour of SGD when the datapoints are sampled with replacement, it weighs every function $i$ uniformly.
Thus it cannot successfully capture the effect of different orders in \eqref{eq:algo} on 
the optimization trajectory, when the functions are sampled non-uniformly.

\begin{example}[Example when \eqref{eq:sigmasgd} fails, \cite{Mohtashami2022data_order}]\label{ex:sigmasgdfail}
As an example, consider $n$ univariate functions ($x \in \R$) defined as $f_i(x) = \frac{1}{2}(x - 1)^2$, 
for the first half: $1 \leq i \leq \frac{n}{2}$, 
and for the second half we define $f_i(x) = \frac{1}{2}(x + 1)^2$, $\frac{n}{2} < i \leq n$, with $n$ being an even number.
Let us consider two instances of algorithm \eqref{eq:algo}.
In the first instance, the method cycles through functions in the initial order $f_1 ,\dots, f_n$ without shuffling,
and in the second instance it cycles in the following order:
$f_1, f_{\frac{n}{2} + 1}, f_2, f_{\frac{n}{2} + 2}, \dots, f_n$.
In the second case the algorithm will converge faster than in the first one. However, by assuming only 
smoothness (\cref{as:smooth}) and the classic boundedness of variance \eqref{eq:sigmasgd}, there is no way to distinguish the two cases, 
as these assumptions do not hold any information about the chosen order. \looseness=-1
\end{example}

This example motivates the need to couple the variance assumption in~\eqref{eq:sigmasgd} 
with the order of the gradients used in the method.\looseness=-1

\paragraph{Variance assumptions that take the data-ordering into account.}

Several recent works
\citep{lu2022a:shuffle, Mohtashami2022data_order, Mishchenko2020RandomRS}
proposed assumptions that depend on the data order.

Most prior works considered correlation periods of a fixed size $\tau = n$. For instance, \citep{Mohtashami2022data_order} define a finite constant $\sigmaepoch$ (assuming it is finite $\sigmaepoch < \infty$), such that
\begin{align} 
     \textstyle\sigmaepoch := \sigman \,.\label{eq:sigmaepoch}
\end{align}
A very similar assumption was used by \cite{lu2022a:shuffle} where they allow the staring point to be an arbitrary index, but the correlation distance $\tau$ can be as large as $T$. %
We note that both these prior works impose an additional assumption that considered data orderings that do not deviate much from the full gradient, while we do not impose such conditions. %
Both of the works considered a refined growth condition, which we do not consider in this work for simplicity.

Let us explain how an epoch based assumption $\sigmaepoch < \infty$ can help in characterising the order of the gradients as compared to 
the classic variance $\sigmasgd$.

\begin{example}
Indeed, in Example~\ref{ex:sigmasgdfail} above we could not distinguish the two different sequences 
using only the variance $\sigmasgd$. 
However, for the first sequence, we can use a bound\looseness=-1
$$
\ba{rcl}
\sigmaepoch\bigl[ (1, 2, \ldots, n) \bigr] &=& \frac{n^2}{8} \sigmasgd,
\ea
$$
while in the second case, we can set
$$
\ba{rcl}
\sigmaepoch\bigl[ (1, \frac{n}{2} + 1, 2, \frac{n}{2} + 2, \ldots, n) \bigr] &=& \frac{1}{2}\sigmasgd.
\ea
$$
\end{example}
Hence, employing the global bound~\eqref{eq:sigmaepoch}
on our sequence variance $\sigma_{k, \tau}^2$ for $\tau = n$, we can clearly quantify the effect
of using different orders of the gradients. As we explain in the next example, it still has some limitations.

\begin{example}[Example when Assumption~\eqref{eq:sigmaepoch} fails]

Consider a \emph{single function} algorithm described in Example~\ref{ex:singlefunc}, i.e.\ $i_t \equiv 1~\forall t$, then the variance of any function $i \neq 1$ is irrelevant, as the algorithm never sees them.
For this sequence, let us define $\sigmaone = \sigma^2_{k,1}$ (for an arbitrary index $k\leq T$). We can observe $\sigma^2_{k,n} = n^2 \sigmaone$. 
However, the number of functions $n$ should not matter for convergence properties of this algorithm, as the algorithm sees only one function. \looseness=-1
\end{example}

\subsection{Our observation on effective correlation time}\label{sec:variance}

However, the main limitation of the prior work and of considering the epoch-based variance $\sigmaepoch$ is not because it cannot capture the convergence behaviour of a single function Example~\ref{ex:singlefunc}. It is because of its limitations when analysing \emph{incremental gradient methods}---methods of the main interest in this work. 

Due to the technical reasons, analysing theoretical convergence of \eqref{eq:algo} using the variance $\sigmat$ with the correlation period $\tau$ puts a constraint on the choice of stepsize as $\gamma < \cO \left(\nicefrac{1}{L \tau}\right)$. \looseness=-1

Setting the correlation period $\tau = n$, as the prior works did \citep{lu2022a:shuffle, Mohtashami2022data_order, Mishchenko2020RandomRS}, therefore limits the stepsize as $\gamma < \cO\left(\nicefrac{1}{L n}\right)$. The small stepsize means small progress at every iteration of algorithm \eqref{eq:algo}, slowing down its convergence behaviour, especially when $n$ is large. Mathematically this results in the $n$ times slow down in the first term of convergence: $\cO\left(\nicefrac{{\color{mydarkred} \boldsymbol{n} }}{\epsilon}\right)$ term in the convergence rate (see Table~\ref{tab:rates}).

In this work, we avoid such a restriction of the stepsize by adaptively choosing the correlation period $\tau$ instead of fixing it to $n$. In our work, the correlation period $\tau$ is not chosen in advance, but it adapts to the stepsize $\gamma$ and the smoothness constant $L$ as $\tau = \Theta\left(\nicefrac{1}{L \gamma}\right)$ so that in the end the stepsize $\gamma$ becomes unrestricted by $\tau$, but instead $\tau$ becomes restricted by the stepsize $\gamma$. This allows us to improve $n$ times the first term of convergence to $\cO\left(\nicefrac{{\color{mydarkgreen} \boldsymbol{1} }}{\epsilon}\right)$ (see Table~\ref{tab:rates}).

\paragraph{Intuitive informal explanation.} Intuitively, effective correlation time $\tau$ should capture how many iterations one needs to perform to move "sufficiently far" from the initial point. This measure should depend only on the stepsize $\gamma$ and the smoothness constant $L$ of functions $f_i$. Indeed, for the larger stepsizes, every step of Algorithm~\eqref{eq:algo} will make the larger progress, and the iterates $\xx_t$ and $\xx_{t  +1}$ will be further apart, reducing the effective correlation time. Similarly, the larger the smoothness constant $L$, the faster the function $f$ changes, and the smaller the distance one need to have to make the iterations uncorrelated for a given $f$ (decreasing effective correlation time). Therefore, $\tau = \Theta\left(\nicefrac{1}{L \gamma}\right)$.

\begin{table}[t]
\centering
\resizebox{0.8\linewidth}{!}{
\begin{tabular}{l|c}
\toprule
Algorithm & Upper bound on $\sigmat$ \\ \midrule
SGD, Ex.~\ref{ex:sgd} & $\tau \sigmasgd$\\
RR, Ex~\ref{ex:RR} & $4 \min\{\tau, n \} n \sigmasgd$ \\
IG and SS, Ex.~\ref{ex:ig}, Ex~\ref{ex:SS} & $\min\{\tau, n\} n \sigmasgd$\\
Single function, Ex.~\ref{ex:singlefunc} & $\tau^2 \sigmaone$\\
\bottomrule
\end{tabular}
}
\caption{Upper bounds on $\sigmat$ for notable special cases.}
\label{tab:sigmas}
\vspace{-5mm}
\end{table}

\section{Main theorem}

In this section we present our main theoretical result.
\begin{theorem}\label{thm:main}
Let each of the functions $f_i$ be $L$-smooth (\Cref{as:smooth}). 
Let the stepsize $\gamma$ in Algorithm~\eqref{eq:algo} satisfy: $\gamma \leq \frac{1}{8 \sqrt{3} L}$. Let $\tau = \bigl\lfloor\frac{1}{8 \sqrt{3} L \gamma} \bigr\rfloor$, and assume that $\sigmat$ from Def.~\ref{def:main_variance} is finite,
for $k = 0, \dots, \big\lfloor \frac{T}{\tau}\big\rfloor$, where
$T$ is a number of iterations. Then,  Algorithm~\eqref{eq:algo} converges at the rate:
\begin{align*}
    \frac{1}{T} \sum_{t = 0}^T \E \norm{\nabla f(\xx_t)}^2 \leq \cO\bigg( \frac{F_0}{\gamma T} + L^2 \gamma^2 \sigmat %
    \bigg)%
    \,,
\end{align*}
where $F_0 = f(\xx_0) - f^\star$. 
\end{theorem}
We give complete proof of this theorem in the Appendix. %

The convergence rate consists of two terms. The first term is the optimization term, and it recovers the tight convergence rate of noiseless gradient descent ($\sigmat = 0$). The second term is the most interesting one, it shows the effect of chosen data orders in Algorithm~\eqref{eq:algo}. The convergence rate depends on the sequence variance $\sigmat$. %

We note that since $\sigmat$ depends on $\tau$, which in itself depends on the stepsize $\gamma$, we cannot directly tune the stepsize in the upper bound.
However, in several notable special cases we show how one can we tune the stepsize by estimating $\sigmat$ (see Table~\ref{tab:sigmas}), and get the final convergence rate. 

\subsection{Implications for the Special Cases}
\paragraph{SGD, Ex.~\ref{ex:sgd}} Since $\sigmat \leq \tau \sigmasgd$ (see Table~\ref{tab:sigmas}), %
and using that $\tau = \Theta\big(\nicefrac{1}{L \gamma}\big)$ the convergence rate in Theorem~\ref{thm:main} converts to 
\begin{align}\label{eq:sgd_rate}
    \textstyle \frac{1}{T} \sum_{t = 0}^T \E \norm{\nabla f(\xx_t)}^2 \leq \cO\big( \frac{F_0}{\gamma T} + L \gamma \sigmasgd \big) \,, &&
\end{align}
with $\gamma \leq \textstyle\frac{1}{12 L}$.
This recovers classical convergence rate of SGD for non-convex functions (up to constants). In particular, after tuning the learning rate, we get the convergence of $\cO\bigl(  \nicefrac{L F_0}{T} + \sqrt{\nicefrac{L F_0 \sigmasgd }{T}}\bigr)$. %

Importantly, the previous works \cite{Mohtashami2022data_order, lu2022a:shuffle} on analysing arbitrary ordered sequences in SGD could not recover the tight convergence rates in these cases. \looseness=-1

\paragraph{Incremental Gradient and Single Shuffle, Ex.~\ref{ex:ig}, Ex.~\ref{ex:SS}} Substituting $\sigmat \leq \min\{\tau, n\} n \sigmasgd$, we get that $\frac{1}{T} \sum_{t = 0}^T \E \norm{\nabla f(\xx_t)}^2$ is smaller than
\begin{align}
\label{eq:ig_rate}
    \textstyle\cO\big( \frac{F_0}{\gamma T} + L \gamma \min\{1, L \gamma n \} n \sigmasgd\big) \,,
\end{align}
with $\gamma \leq \frac{1}{12 L}$. 
Thus, after the stepsize tuning, we get the rate of $\cO \left( \nicefrac{L  F_0}{T} + \min \left\{ \left(\nicefrac{L F_0 n \sigma_{\operatorname{SGD}} }{T}\right)^{\frac{2}{3}}, \sqrt{\nicefrac{L F_0 n \sigmasgd}{T}} \right\}\right)$ which is strictly better for both of the terms than the previously best known bound $\cO \left( \nicefrac{L F_0 n}{T} + \left(\nicefrac{L F_0 n\sigma_{\operatorname{SGD}}}{T}\right)^{\frac{2}{3}} \right)$ \citep{Mohtashami2022data_order, lu2022a:shuffle, Mishchenko2020RandomRS}.\looseness=-1

\paragraph{Single Function, Ex.~\ref{ex:singlefunc}} We now show that our convergence rate can give tight guarantees when the optimization is always performed on a single function $i_t \equiv 1$. As discussed above, in this case we have $\sigmat = \tau^2 \sigma_1^2$, and thus we can guarantee the convergence of
\beq \label{eq:single_func_rate}
\ba{rcl}
        \frac{1}{T} \sum_{t = 0}^T \E \norm{\nabla f(\xx_t)}^2 
        & \!\! \leq \!\! &
        \cO\big( \frac{F_0}{\gamma T} + \sigmaone\big)\,, 
\ea
\eeq
with $\gamma \textstyle\leq \frac{1}{12 L}$.
It provides the method with convergence only to the neighbourhood of size $\sigmaone$, which cannot be reduced.
This is expected, as the algorithm has no way to learn information about the full function~$f$ \citep{ajalloeian2020convergence}. The same convergence rate could be simply obtained by noticing that $\norm{\nabla f(\xx_t)}^2 \leq  2 \norm{\nabla f_1(\xx_t)}^2 + 2 \norm{\nabla f(\xx_t) - \nabla f_1(\xx_t)}^2 \leq 2 \norm{\nabla f_1(\xx_t)}^2 + 2 \sigmaone$ for all $\xx_t$ and that $\frac{1}{T} \sum_{t = 0}^T \norm{\nabla f_1(\xx_t)}^2 \leq \cO\left( \nicefrac{F_0}{\gamma T}\right)$ as the algorithm performs a full gradient descent on function~$f_1$ that is $L$-smooth. In the Appendix we show that this convergence rate is in fact tight.  %

This example serves to highlight that our Theorem~\ref{thm:main} can capture tightly the cases even when the selected orders $i_t$ result in the gradients that are always far from the true gradient $\nabla f(\xx)$. \looseness=-1

\subsection{Random Reshuffling}
Algorithm~\eqref{eq:algo} and our Theorem~\ref{thm:main} cover arbitrary orderings of the data, including random reshuffling (RR) algorithm that cycles through the dataset, reshuffling the order of the data after every epoch (Ex.~\ref{ex:RR}).

Applying the bound on $\sigmat$ from Table~\ref{tab:sigmas}, and tuning the stepsize $\gamma$ the convergence rate in Theorem~\ref{thm:main} converts to
\begin{align*}
    \cO \left( \nicefrac{L  F_0}{T} + \min \left\{ \left(\nicefrac{L F_0 n \sigma_{\operatorname{SGD}} }{T}\right)^{\frac{2}{3}}, \sqrt{\nicefrac{L F_0 n \sigmasgd}{T}} \right\}\right)
\end{align*}
The prior best known convergence rate of random reshuffling algorithm is \citep{Mohtashami2022data_order, lu2022a:shuffle, Mishchenko2020RandomRS}
$
    \cO \left( \nicefrac{L  F_0 {\color{mydarkred}\boldsymbol{n}}}{T} +  \left(\nicefrac{L F_0 {\color{mydarkgreen} \boldsymbol{\sqrt{n}}} \sigma_{\operatorname{SGD}} }{T}\right)^{\frac{2}{3}}\right),
$
which is $n$ times worse than our rate in the fist term, however it is $\sqrt{n}$ better in the leading stochastic term. It remains an open question whether it is possible to remove $n$ from the first term of convergence without negatively affecting the leading stochastic term.\looseness=-1

\begin{figure*}[h!]
    \vspace{-0.8em}
    \centering
    \includegraphics[width=1.0\linewidth]{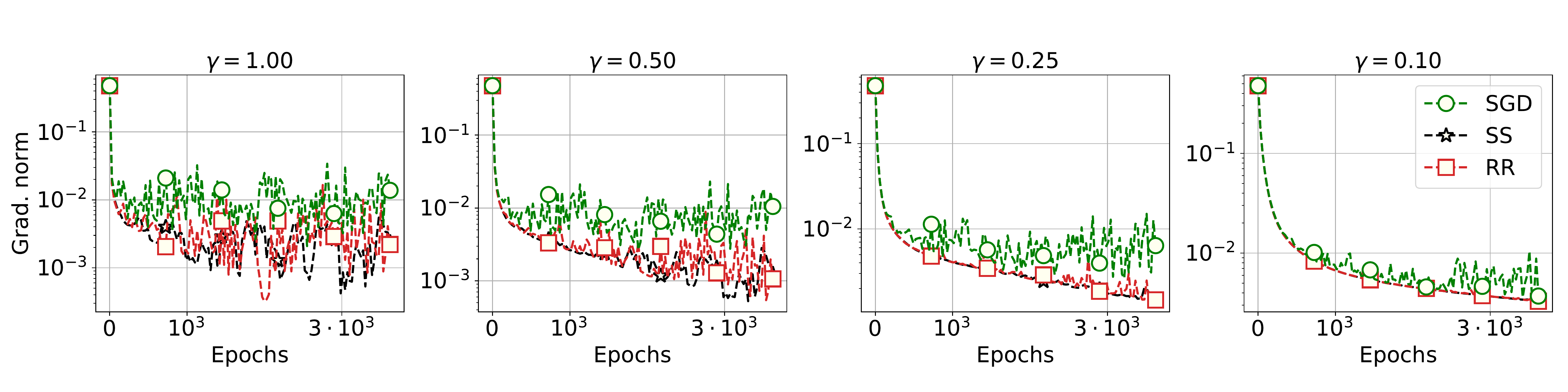}
    \vspace{-2em}
    \caption{\small Convergence curves for logistic regression on the \texttt{Australian} dataset \cite{chang2011libsvm}.
    Random Reshuffling (RR) and Single Shuffling (SS) are faster than SGD across varying learning rates.}
    \label{fig:LogReg}
\end{figure*}
\section{Experiments}
\label{SectionExperiments}

In this section, we present illustrative numerical experiments comparing different strategies for selecting stochastic gradients: SGD (sampling gradients with replacement), 
Single Shuffle (SS, using one random permutation for all epochs),
and Random Reshuffling (RR, generating permutation for each new epoch).
We demonstrate that both of shuffle strategies are not only beneficial due to simpler and faster implementations, but also achieve comparable or even better convergence than plain SGD. %

\paragraph{Quadratic objectives.}

We first consider synthetic quadratic functions of the form 
$f_i(\xx) = 
\frac{1}{2} \la \mA \xx, \xx\ra - \la \bb, \xx \ra + \la \uu_i, \xx \ra$, for
$1 \leq i \leq n$,
where $\mA = \mA^{\top} \succeq 0$ is a given matrix and $\bb \in \R^d$ is a fixed vector,
while vectors $\{ \uu_i \}_{i = 1}^n$ are generated 
randomly and rescaled to be zero-mean $\sum_{i = 1}^n \uu_i  =  \0$, 
and have the desired variance $\sigmasgd = 0.01$.\looseness=-1

In \cref{fig:quadr_converg} we depict convergence curves for various learning rates $ \gamma \in \{5 \cdot 10^{-2}, 10^{-2} , 5 \cdot 10^{-1}\}$. In all the settings, both of the shuffle variants (SS and RR) outperform classical SGD, with the improvement being larger when the stepsize is smaller. This correspond to our theory in \cref{eq:ig_rate}, as the smaller $\gamma$, the larger the difference between $L \gamma \min\{1, L \gamma n\} n \sigmasgd$ (incremental gradient rate) and $L \gamma \sigmasgd $ (SGD rate). 

\begin{figure*}[t!]\vspace{-2mm}
    \centering
    \includegraphics[width=1.0\linewidth]{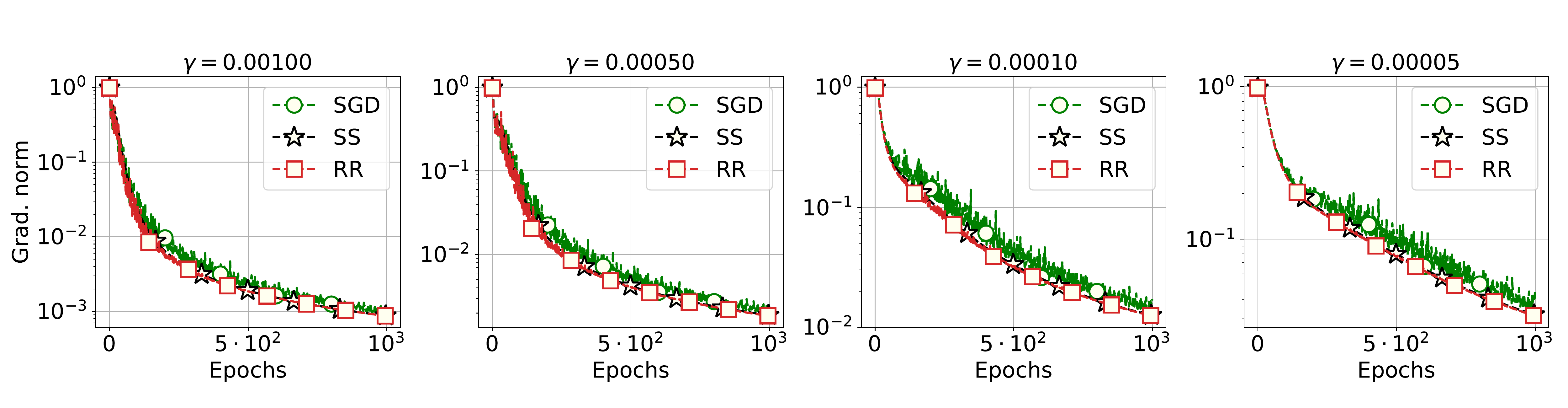}
    \vspace{-2em}
    \caption{\small Training 
    the neural network model on a subset of MNIST dataset of size $1000$.
    Random Reshuffling (RR) and Single Shuffling (SS) are better than SGD. }
    \label{fig:NN}
\end{figure*}
\begin{figure*}[h!]\vspace{-2mm}
    \centering
    \includegraphics[width=1.0\linewidth]{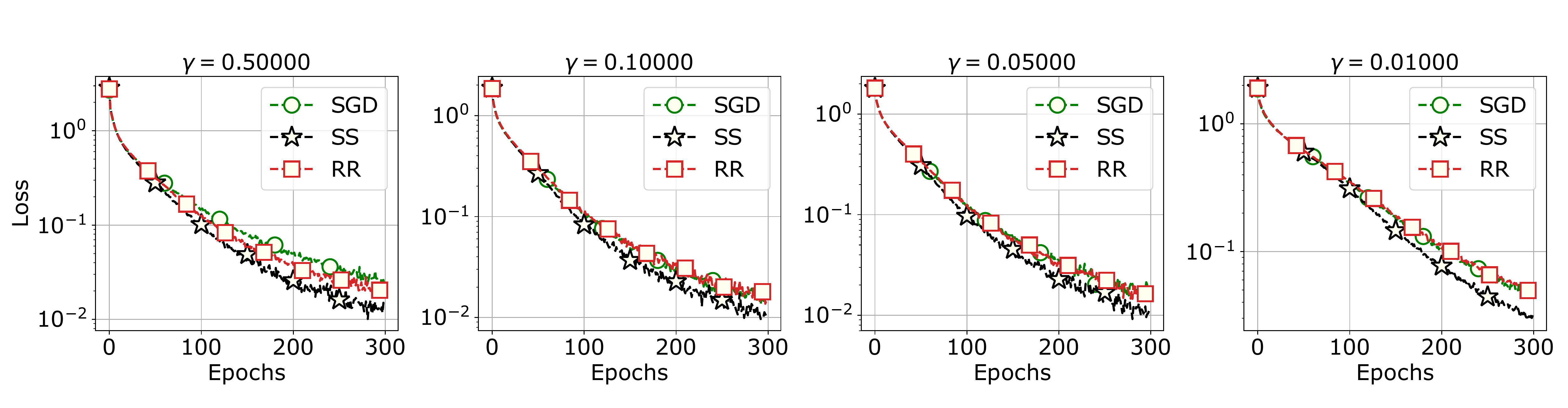}\vspace{-2em}
    \caption{\small Training the neural network model on CIFAR dataset. Single Shuffle (SS) shows the best performance.}
    \label{fig:Cifar}\vspace{-2mm}
\end{figure*}

In \cref{fig:quadr_n} we additionally investigate the convergence properties of SS, RR and SGD across varying the number of components $n$ in \eqref{eq:problem}. We fix the target gradient norm to $\epsilon = 10^{-2}$ and measure how many iterations $T_{\min} (n, \epsilon)$ it takes for each of the methods to achieve this target accuracy. We tune the stepsize over the fixed grid separately for each method, and for each $n$. We repeat each experiments $30$ times %
and plot the mean and 95\% confidence intervals for $T_{\min}(n, \varepsilon)$. \looseness=-1

We can see that both RR and SS \textit{outperform} the classical SGD
and the improvement is the most significant for smaller values of $n$.\looseness=-1 %

\paragraph{Logistic regression.}

In \cref{fig:LogReg}, we consider the problem of training the logistic regression model on machine learning datasets from \citep{chang2011libsvm}\footnote{\href{https://www.csie.ntu.edu.tw/~cjlin/libsvmtools/datasets/}{www.csie.ntu.edu.tw/~cjlin/libsvmtools/datasets/}.}. 
The objective components are convex and of the following form
$f_i(\xx) = \log(1 + e^{ - y_i \la \aa_i, \xx \ra } )$
where $\{ \aa_i \}_{i = 1}^n$ are the feature vectors from $\R^d$ and $\{ y_i \}_{i = 1}^n$ are the labels from $\{ \pm 1 \}$. 
We compare RR, SS and SGD across four different learning rates $\gamma \in \{1, 0.5, 0.25, 0.1\}$.
The results are shown in Fig.~\ref{fig:LogReg} and in the Appendix.
We see that the Random Shuffle (RR) and Single Shuffle (SS)
methods show significantly better performance 
than SGD, in all considered cases.
Therefore, in practice, it seems indeed reasonable to always use 
the shuffle strategies instead of the classical sampling with 
replacement.\looseness=-1

\paragraph{Neural network training.}

We train a three-layer neural network
(one convolutional layer and two fully-connected layers with $\tanh$ activation functions)
with the total number of parameters $d = 140697$.
For our data we use a random subset of MNIST dataset of size $1000$.
After each epoch, we evaluate the norm of the full gradient 
for the entire model. The result are shown in Fig.~\ref{fig:NN}.
We see that the Random Reshuffling (RR) 
works at least as well as SGD or even better. The smaller the learning rate $\gamma$, the larger the improvement as predicted by our theory.
We see that the convergence of RR is also more stable.\footnote{In the Appendix, we provide full details on our experimental setup, computational environment, network architectures, as well as additional experiments.}

Finally, we show experimental results on CIFAR dataset with \texttt{resnet18}
architecture \cite{he2016deep}
(the total number of parameters is $d = 11181642$).
In our training, we sample batches of a fixed size $256$,
comparing different strategies of sampling the data (SGD with replacements, Single Shuffle,
and Random Reshuffling). The results are shown in
Figure~\ref{fig:Cifar}.
We see that both shuffling strategies are better than the plain SGD. The best performance is achieved by Single Shuffle (SS).

\section{Conclusion}
We present a framework for analyzing SGD algorithms under arbitrary data orderings.
For incremental gradient descent and single shuffle algorithms we improve the previously best known convergence rates. 
This improvement is in part 
because our framework allows to chose larger stepsizes than in previous analyses because we can consider shorter correlation periods.
Our study highlights the benefits of using SGD with single shuffling and provides new insights into its convergence properties for non-convex smooth optimization.

\section*{Acknowledgments}
AK would like to thank Zachary Charles for the useful discussions that lead to the idea of this project. We also thank Amirkeivan Mohtashami for discussions, as well as Vinitra Swamy, Mary-Anne Hartley, and Alexander Hägele for their suggestions on manuscript writing. We are grateful to Cristobal Guzman, Grigory Malinovsky, and Peter Richtárik
for their useful comments on the first version of our manuscript.
AK is supported by a Google PhD Fellowship.
ND is supported by the Swiss State Secretariat for Education, Research and Innovation (SERI)
under contract number 22.00133.

\section*{Impact Statement} 
Our work is focused on improving the theoretical understanding of existing algorithms, we do not feel specific societal consequences must be specifically highlighted here.

\nocite{langley00}

{%
\bibliography{references}
}
\bibliographystyle{icml2024}

\newpage
\appendix
\onecolumn
\section{Additional Experiments}

\subsection{Variance Estimation}

First, we present an empirical study of different variance parameters
(see Section~\ref{sec:variance}),
that are used to bound the convergence rates of stochastic gradient descent methods.
For that, we take the \texttt{w1a} dataset
from LibSVM \citep{chang2011libsvm} and consider the set of gradients
$\{ \nabla f_i(\xx_0) \}_{i = 1}^n$
for the Logistic Regression model \eqref{AppLogReg},
evaluated at a fixed point $\xx_0 = \0$. We have an estimate for the standard variance
used in the classic SGD analysis:
$$
\ba{rcl}
\sigma_{\text{SGD}}^2 & := & \frac{1}{n} \sum\limits_{j = 1}^n \| \nabla f_j(\xx_0) - \nabla f(\xx_0) \|^2.
\ea
$$
Then, for a randomly (uniformly) selected data ordering $(i_1, \ldots, i_n)$, we compute the quantities:
$$
\ba{rcl}
\xi_\tau\Bigl[ (i_1, \ldots, i_n) \Bigr]
& := &
\Bigl\| \sum\limits_{t = 0}^{\tau} (\nabla f_{i_t}(\xx_0) - \nabla f(\xx_0) )  \Bigr\|^2,
\qquad 1 \leq \tau \leq n,
\ea
$$
and use them for estimating our variance parameter, as follows:
$$
\ba{rcl}
\sigma_{\tau}^2 & := & \max\limits_{j = 0, \ldots, \tau - 1} \bar{\xi}_j,
\ea
$$
where $\bar{\xi}_j$ is the empirical mean estimated with $100$ samples
of data orderings.
The result of our computations is shown in Figure~\ref{fig:Sigma}.
We see that the value of $\sigma_{\tau}^2$, that is employed in our analysis
of the Shuffle SGD methods is significantly better than its counterpart $n \sigma_{\text{SGD}}^2$
from the classic SGD with replacement.
Hence, these observations empirically confirm superiority
of the Shuffle strategies in practice.

\begin{figure}[h!]
    \centering
    \hspace*{-20pt}
    \includegraphics[width=0.55\linewidth]{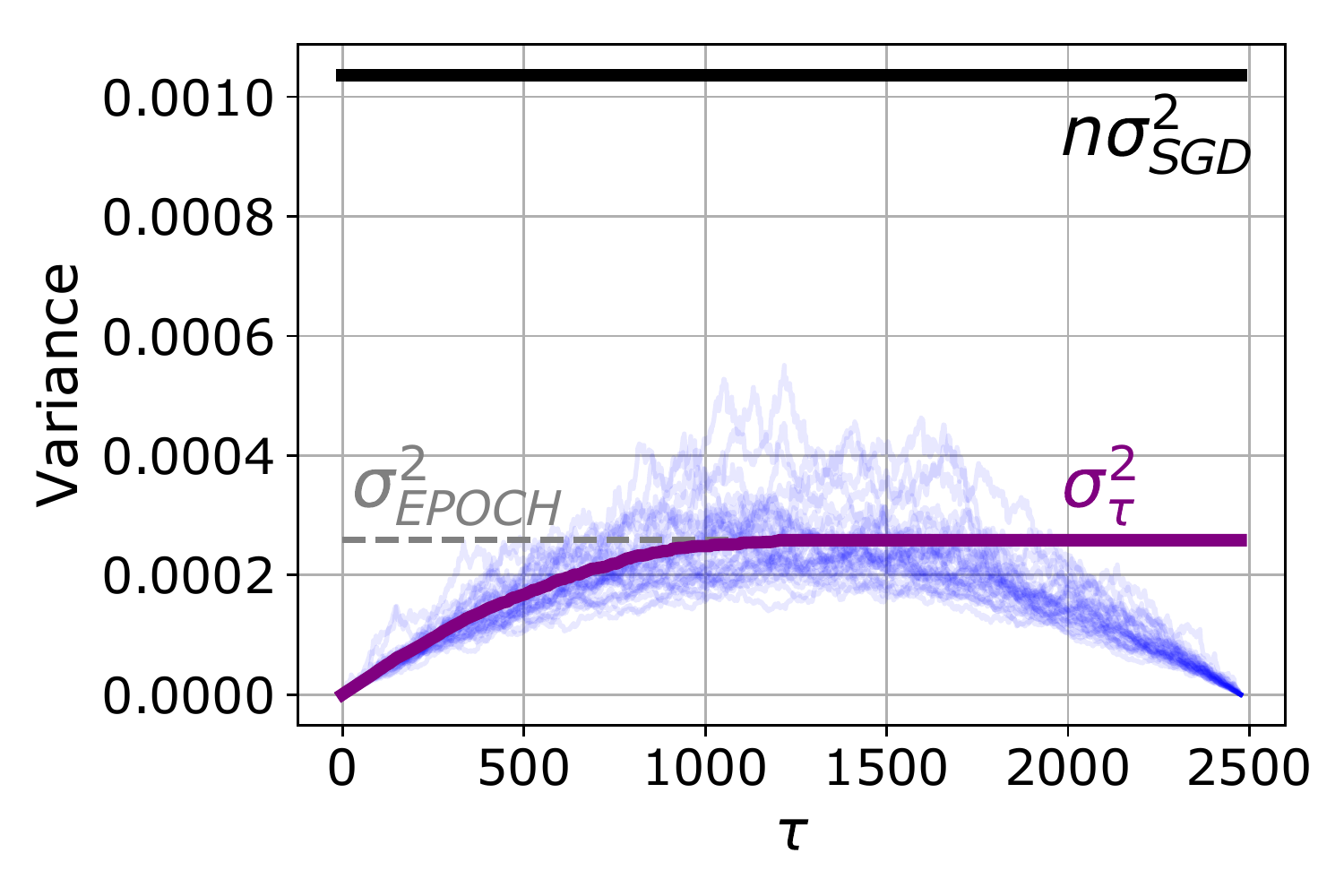}
    
    \caption{\small Variance estimation for the Logistic Regression
    model on \texttt{w1a} dataset. Our variance parameter $\sigma_{\tau}^2$ is significantly
    better than its corresponding upper bound $n \sigma_{\text{SGD}}^2$ used in classical SGD as well as in some prior works of analysing Shuffle SGD strategies \citep{Mishchenko2020RandomRS}. When $\tau$ is smaller than $\frac{n}{2}$ our variance parameter $\sigma_{\tau}^2$ is also better than $\sigmaepoch$ used to analyse Shuffle SGD in \citep{Mohtashami2022data_order, lu2022a:shuffle}.}
    \label{fig:Sigma}
\end{figure}

\subsection{Logistic Regression}

In this section we present experimental comparison of SGD with Single Shuffling and Random Reshuffling for training Logistic Regression model on several other datasets from LibSVM \citep{chang2011libsvm}.
Our objective has the following form,
\beq \label{AppLogReg}
\ba{rcl}
\min\limits_{\xx \in \R^d}
\Bigl[
f(\xx) & := & \frac{1}{n}\sum\limits_{i = 1}^n
\log(1 + e^{-y_i \la \aa_i, \xx \ra} )
\Bigr],
\ea
\eeq
where $\{ \aa_i\}_{i = 1}^n$ are the vectors of features
from $\R^d$ and $\{ y_i \}_{i = 1}^n$ are the training labels from $\{ \pm 1 \}$.
Thus, our objective is smooth and convex, but not strongly convex.
We apply all the methods starting from $\xx_0 = \0$
and using a constant stepsize $\gamma > 0$.
We vary several values for $\gamma$, that are better suitable for each
particular dataset. We show how the full gradient norm $\| \nabla f(\xx_t) \|$
changes with iterations $t \geq 0$. One Epoch is equal to $n$ iterations.
The datasets we use are:
\texttt{ionosphere} $(d = 34, n = 351)$,
\texttt{breast-cancer} $(d = 10, n = 683)$,
\texttt{a9a} $(d = 123, n = 32561)$, randomly shrinked \texttt{w1a} $(d = 300, n = 500)$, \texttt{rcv1} $(d = 47236, n = 20242)$.
The methods are implemented in Python 3.

In Figure~\ref{fig:LogReg_extra}, we see that the methods with
Random Reshuffling (RR) and Single Shuffling (SS)
constantly demonstrate better convergence behaviour for this problem,
as compared to the classical SGD that uses sampling with replacement.
For a fixed value $\gamma$, RR and SS strategies
provide the method
with a smaller variance of stochastic gradients,
which in turn results in a more stable and faster convergence.
This confirms our theory.

\subsection{Neural Networks}

In this section, we present our experimental results
on training Neural Network models 
for MNIST~\cite{lecun1998mnist} and CIFAR~\cite{krizhevsky2009learning} datasets.
For MNIST, we used a small simple architecture
consisting of one convolutional layer
(with $3$ output channels and the kernel of size $2$)
and two consequent fully connected layers
with $64$ and $10$ output neurons respectively.
Therefore, in total our model has $d = 140697$
parameters to train. We used smooth \texttt{tanh} activation
function between layers.
After each Epoch, we compute the full gradient norm
evaluated over the entire training data.
The results for a random selection of MNIST dataset of size
$n = 1000$ are shown in Figure~\ref{fig:NN}
and for $n = 500$ in Figure~\ref{fig:NN_extra}.

We see that all the methods posses similar convergence rates,
while the convergence behaviour of RR and SS strategies
is more stable than that of the SGD that uses sampling with replacement.
\begin{figure}[h!]
    \centering
    \includegraphics[width=1.0\linewidth]{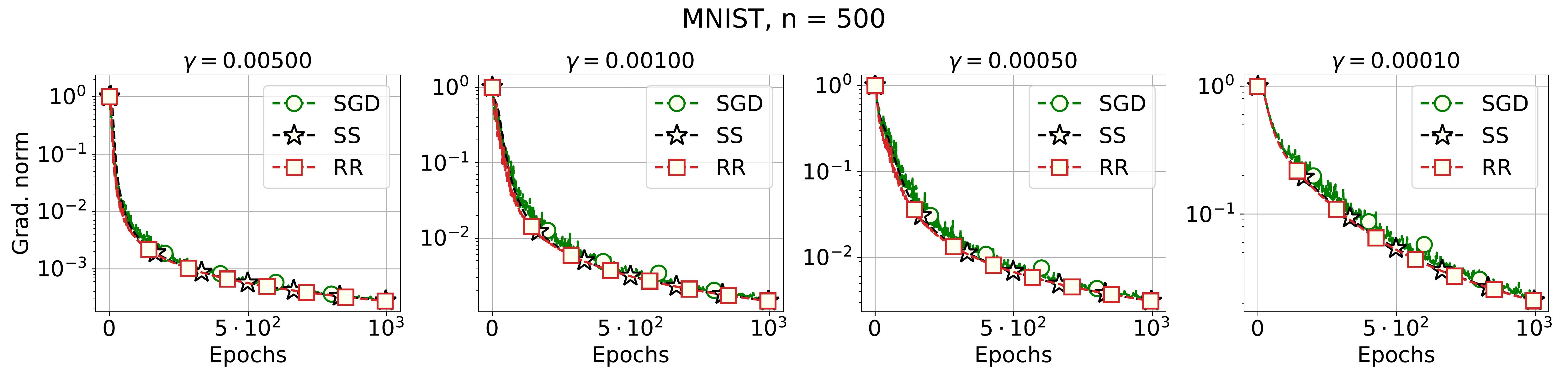}
    \caption{\small Training 
    the neural network model on MNIST dataset.
    Random Reshuffling (RR) works always the same or better than SGD.}
    \label{fig:NN_extra}
\end{figure}

\begin{figure}[h!]
    \centering
    \includegraphics[width=0.97\linewidth]{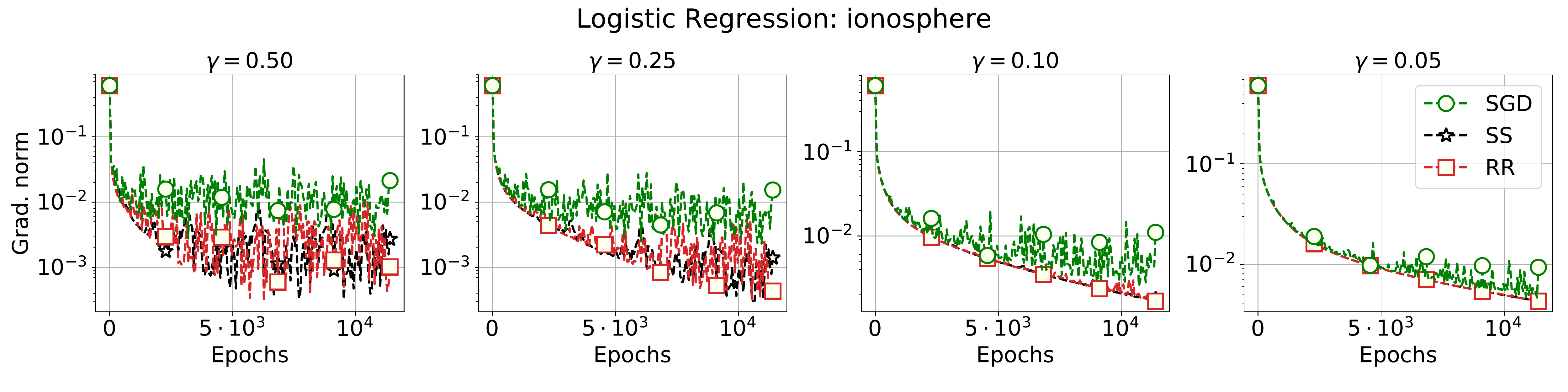}
    \includegraphics[width=0.97\linewidth]{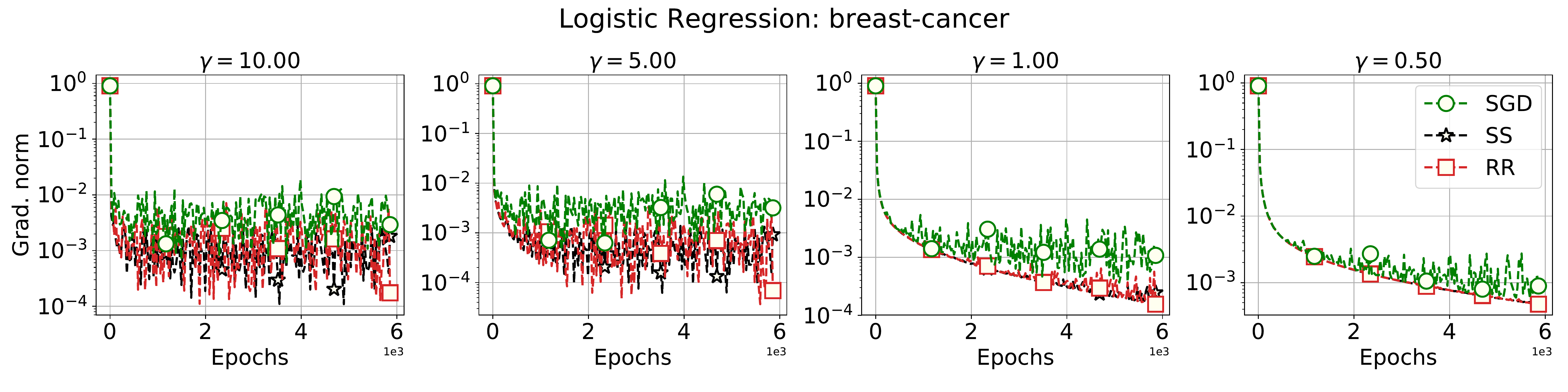}
        \includegraphics[width=0.97\linewidth]{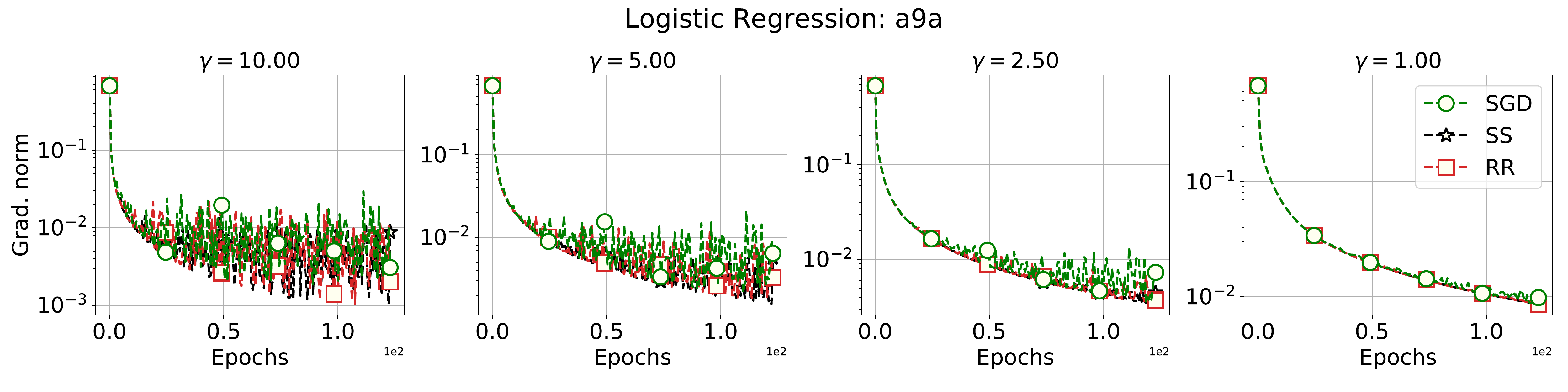}
        \includegraphics[width=0.97\linewidth]{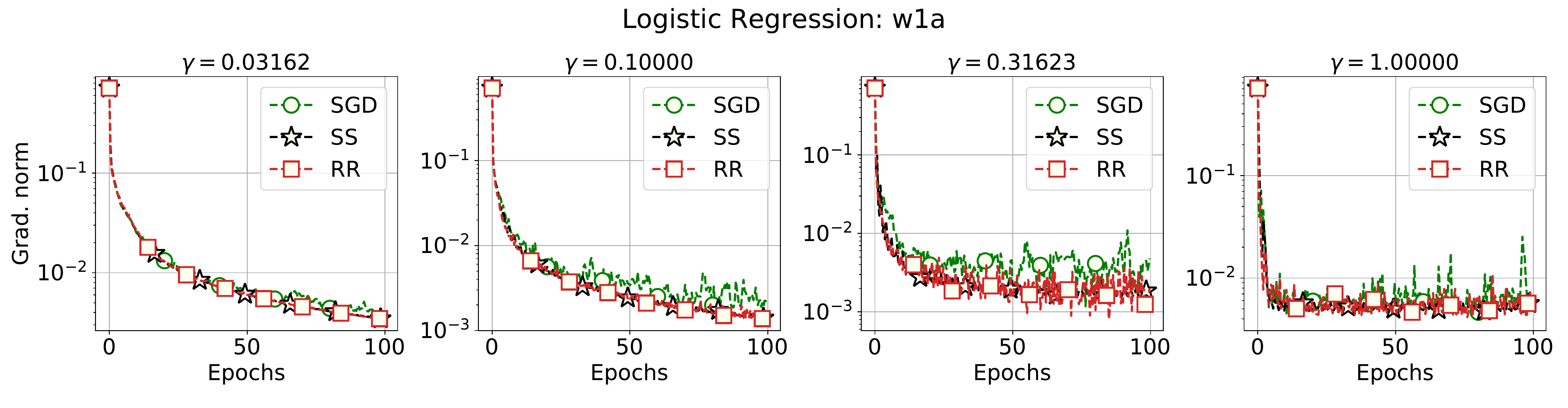}
        \includegraphics[width=0.97\linewidth]{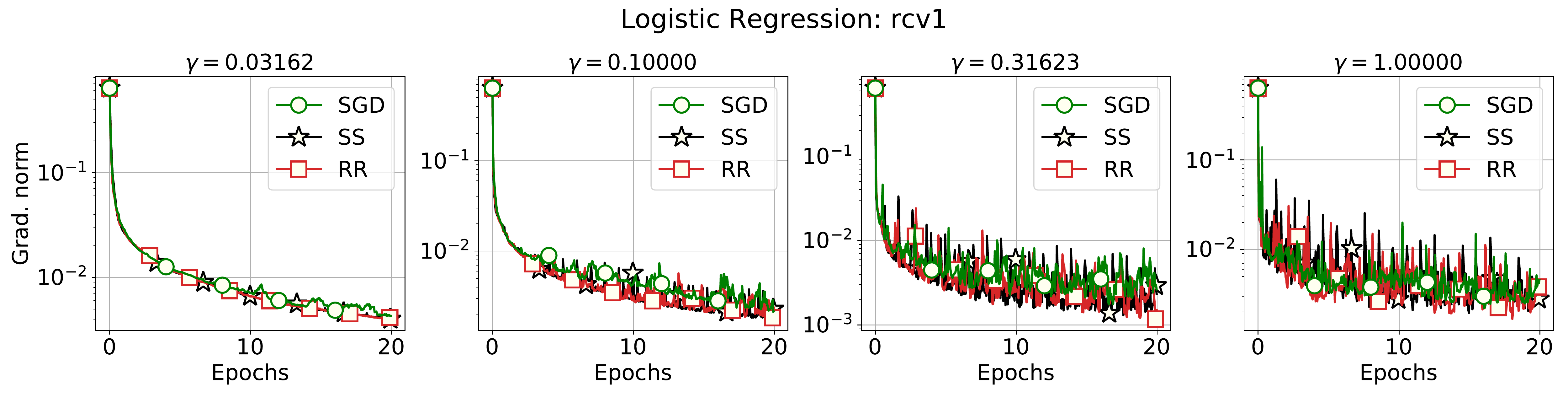}
    \caption{\small Convergence curves for logistic regression on the real data.
    Random Reshuffling (RR) and Single Shuffling (SS) work always better than SGD across varying learning rates.}
    \label{fig:LogReg_extra}
\end{figure}

\clearpage
\section{Proofs}
We first restate L-smoothness condition from \cref{as:smooth}.

\begin{align}\label{eq:l-smooth}
    \norm{\nabla f_i(\xx) - \nabla f_i(\yy)} \leq L \norm{\xx - \yy}\,, \forall \xx,\yy \in \R^d\,.   
\end{align}

\subsection{Useful inequalities}
\begin{lemma}
For any finite set of vectors $\{\aa_i\}_{i = 1}^{n} \subset \R^d$,
\begin{align}\label{eq:sum_of_n_vectors}
    \norm{\sum_{i = 1}^n \aa_i}^2 \leq n \sum_{i = 1}^n \norm{\aa_i}^2.
\end{align}
\end{lemma}
\begin{lemma}
For any two vectors $\aa, \bb \in \R^d$ and for all $\alpha > 0$,
\begin{align}\label{eq:scalar_product_ab}
    2 \langle \aa, \bb \rangle \leq \alpha \norm{\aa}^2 + \alpha^{-1}\norm{\bb}^2.
\end{align}
\end{lemma}

\subsection{Main Lemma}
Our proof is based on a technique called \emph{perturbed iterate analysis} that analyzes a sequence of virtual iterates~\citep{mania17:perturbed_analysis,stich21:error-feedback}.
Recently, \cite{koloskova23:correlated_noise} proposed a modified virtual sequence with restart iterations and conducted an analysis for gradients perturbed with constant noise patterns (independent of $\xx_t)$.
Here, we extend their analysis for arbitrary noise perturbations (that can depend on the iterates $\xx_t$).

For our analysis we use the restart virtual sequence $\{ \tilde{\xx}_t \}_{t \geq 0}$,
starting from $\tilde{\xx}_0 = \xx_0$ and defined as follows:
\begin{align}\label{eq:restart}
    \begin{aligned}
        \tilde{\xx}_{t + 1} &= \tilde \xx_t - \gamma \nabla f(\xx_t) \\\tilde \xx_{t + 1} &= \xx_{t + 1}
    \end{aligned}
    &&
    \begin{aligned}
        &\text{if }~(t + 1) \mymod \tau \neq 0 \\ &\text{if }~(t + 1) \mymod \tau = 0
    \end{aligned}
\end{align}
where $\tau = \Theta\bigl( \frac{1}{L \gamma}\bigr)$ is our key parameter.
Note that in \eqref{eq:restart} we use the \textit{full gradients} $\nabla f(\cdot)$
evaluated at the iterates of our Algorithm~\ref{eq:algo}.
We denote by $r(t)$ the closest restart iteration to $t$, i.e. 
$$
\ba{rcl}
r(t) & = & \lfloor \frac{t}{\tau}\rfloor \tau \;\; = \;\; t - t \mymod \tau.
\ea
$$

For simplifying the presentation, we also denote, for $0 \leq t \leq T$:
\begin{align}\label{eq:phi}
    \phi_t(\xx) \;\, = \;\, \norm{\sum_{j = r(t)}^{t} \left(\nabla f(\xx) - \nabla f_{i_j} (\xx)\right) }^2
\end{align}
and
\begin{align}\label{eq:phi_bar}
    \bar{\phi}_t(\xx) \;\, = \;\,
    \begin{cases}
        0, \quad & \text{if} \;\, t = r(t), \\[5pt]
        \displaystyle
        \norm{\sum\limits_{j = r(t)}^{t - 1} \left(\nabla f(\xx) - \nabla f_{i_j} (\xx)\right) }^2, \quad & \text{otherwise}.
    \end{cases}
\end{align}

Since, for $r(t) \leq i \leq r$, we have $r(i) = r(t)$, we conclude that
$$
\ba{rcl}
\phi_i(\xx) & = & 
\displaystyle
\norm{\sum\limits_{j = r(t)}^{i} \left(\nabla f(\xx) - \nabla f_{i_j} (\xx)\right) }^2,
\qquad r(t) \leq i \leq t,
\ea
$$
and, in particular, $\bar{\phi}_t(\xx) = \phi_{t - 1}(\xx)$ wherever $t \not= r(t)$.

Note that due to \cref{def:main_variance}, we have, 
for all $0 \leq t \leq T$:
\beq \label{ExpPhiBound}
\ba{rcl}
\E \phi_t(\xx) & \leq & \sigmakt, \qquad \text{with} \qquad
k = \frac{r(t)}{\tau} = \lfloor \frac{t}{\tau}\rfloor.
\ea
\eeq

First, we prove a lemma to bound the distance between the virtual sequence $\tilde \xx_{t + 1}$ and the real sequence $\xx_{t + 1}$
from Algorithm~\eqref{eq:algo}.
Since the virtual and real iterate sequences have the following updates:
\begin{align} \label{VirtualRealSeq}
    \xx_{t + 1} &= \xx_{r(t)} - \gamma \sum_{j = r(t)}^{t} \nabla f_{i_j}(\xx_j), 
    && \tilde \xx_{t + 1} = \xx_{r(t)} - \gamma \sum_{j = r(t)}^{t} \nabla f(\xx_j),
\end{align}
it holds:
$$
\ba{rcl}
\norm{\tilde \xx_{t + 1} - \xx_{t + 1}}^2 & = & 
\displaystyle
\gamma^2\norm{\sum\limits_{j = r(t)}^{t}\left( \nabla f(\xx_j) - \nabla f_{i_j} (\xx_j) \right) }^2.
\ea
$$
Hence, it is enough just to bound the expression in the right hand side.

\begin{lemma}\label{lem:consensus} Under the same assumptions as in Theorem~\ref{thm:main}, we have, for any $0\leq t \leq T$:
\beq \label{MainLemmaEq}
\ba{rcl}
    \displaystyle
    \norm{ \sum\limits_{j = r(t)}^{t}\left( \nabla f(\xx_j) - \nabla f_{i_j} (\xx_j) \right) }^2
    & \!\!\! \leq \!\!\! & 
    \displaystyle
    3 \phi_t(\xx_{r(t)}) +  48 \gamma^2 L^2 \tau \sum\limits_{j = r(t)}^{t}\phi_j(\xx_{r(t)}) \\[10pt]
    & & \qquad
    \displaystyle
    + 16 \gamma^2 \tau^3 L^2  \sum\limits_{j = r(t)}^{t}\norm{\nabla f(\xx_j)}^2.
\ea
\eeq
Therefore, substituting that $\gamma \leq \frac{1}{8 \sqrt{3} L \tau}$, we have
\beq \label{XtXbartBound}
\ba{rcl}
    \displaystyle
    \norm{ \sum\limits_{j = r(t)}^{t}\left( \nabla f(\xx_j) - \nabla f_{i_j} (\xx_j) \right) }^2
    & \leq &
    \displaystyle
    3 \phi_t(\xx_{r(t)}) +  \frac{1}{4 \tau}  
    \sum\limits_{j = r(t)}^{t}\phi_j(\xx_{r(t)}) \\[10pt]
    & & \qquad 
    \displaystyle
+ \frac{\tau }{12}  \sum\limits_{j = r(t)}^{t}\norm{\nabla f(\xx_j)}^2.
\ea
\eeq
\end{lemma}
\begin{proof}
Indeed, we have
$$
\ba{rcl}
&&\!\!\!\!\!\!\!\!\!\!\!\!\!\!\!\!\!\!\!\!\!\!\!\!\!\!\!\!\!\!\!\!\!\!\!\!\!
\displaystyle
\norm{\sum\limits_{j = r(t)}^{t}\left( \nabla f(\xx_j) - \nabla f_{i_j} (\xx_j) \right) }^2 \\
\\
&\stackrel{\eqref{eq:sum_of_n_vectors}}{\leq}&
\displaystyle
3 \norm{\sum\limits_{j = r(t)}^{t} \left(\nabla f(\xx_{r(t)}) - \nabla f_{i_j} (\xx_{r(t)}) \right)}^2 
+ 3 \norm{\sum\limits_{j = r(t)}^{t} \left(\nabla f(\xx_{r(t)}) - \nabla f(\xx_j) \right) }^2 \\
\\
& & 
\displaystyle
+ \;\; 3 \norm{\sum\limits_{j = r(t)}^{t} \left( \nabla f_{i_j} (\xx_j) - \nabla f_{i_j}(\xx_{r(t)}) \right)}^2  \\
\\
&\stackrel{\eqref{eq:phi}, \eqref{eq:sum_of_n_vectors}, \eqref{eq:l-smooth}}{\leq}& 
\displaystyle
3 \phi_t(\xx_{r(t)}) + 6 \tau L^2  \sum\limits_{j = r(t)}^{t} \norm{\xx_{r(t)} - \xx_j}^2.
\ea
$$
We further need to estimate the term
\begin{align*}
\sum_{j = r(t)}^{t} \norm{\xx_{r(t)} - \xx_j}^2.
\end{align*}
Looking individually at each element of the sum and using the update rule \eqref{eq:algo}, we obtain
\begin{align*}
& \norm{\xx_j - \xx_{r(t)} }^2 \\[5pt]
&\;\;\;\;\;\,=\;\;\;\;\;\, 
\gamma^2 \norm{\sum_{l = r(t)}^j \nabla f_{i_l}(\xx_l) }^2 \stackrel{\eqref{eq:sum_of_n_vectors}}{\leq}  2 \gamma^2 \norm{\sum_{l = r(t)}^j \left( \nabla f_{i_l}(\xx_l) - \nabla f(\xx_l) \right)}^2  
+ 2 \gamma^2 \norm{\sum_{l = r(t)}^j \nabla f(\xx_l)}^2 \\[5pt]
&\;\;\;\;\;\stackrel{\eqref{eq:sum_of_n_vectors}}{\leq}\;\;\;\; 
6 \gamma^2 \norm{\sum_{l = r(t)}^j \left(\nabla f_{i_l}(\xx_{r(t)}) - \nabla f(\xx_{r(t)}) \right)}^2 
+ 6 \gamma^2\norm{\sum_{l = r(t)}^j \left( \nabla f_{i_l}(\xx_{r(t)}) - \nabla f_{i_l}(\xx_{l}) \right)}^2 \\[5pt]
&\qquad\qquad+ \;\; 6 \gamma^2 \norm{\sum_{l = r(t)}^j \left( \nabla f(\xx_{r(t)}) - \nabla f(\xx_{l}) \right) }^2  
+ 2 \gamma^2 \norm{\sum_{l = r(t)}^j \nabla f(\xx_l)}^2 \\[5pt]
&\stackrel{\eqref{eq:phi}, \eqref{eq:sum_of_n_vectors}, \eqref{eq:l-smooth}}{\leq} 6 \gamma^2 \phi_j(\xx_{r(t)}) + 12 \gamma^2 \tau L^2 \sum_{l = r(t)}^j \norm{\xx_{r(t)} - \xx_l}^2 + 2 \gamma^2 \tau \sum_{l = r(t)}^j \norm{\nabla f(\xx_l)}^2\\
&\;\;\;\;\;\stackrel{j \leq t}{\leq}\;\;\;\; 
6 \gamma^2 \phi_j(\xx_{r(t)}) + 12 \gamma^2 \tau L^2 \sum_{l = r(t)}^{t} \norm{\xx_{r(t)} - \xx_l}^2 + 2 \gamma^2 \tau \sum_{l = r(t)}^{t} \norm{\nabla f(\xx_l)}^2.
\end{align*}
Thus,
\begin{align*}
& \sum_{j = r(t)}^{t} \norm{\xx_{r(t)} - \xx_j}^2 \\
& \quad \leq \quad 6 \gamma^2 \sum_{j = r(t)}^{t}\phi_j(\xx_{r(t)}) + 12 \gamma^2 \tau^2 L^2 \sum_{j = r(t)}^{t} \norm{\xx_{r(t)} - \xx_j}^2 
+ 2 \gamma^2 \tau^2 \sum_{j = r(t)}^{t}\norm{\nabla f(\xx_j)}^2.
\end{align*}
Using that $\tau \leq \frac{1}{8 \sqrt{3} L \gamma}$, we get that $12 \gamma^2 \tau^2 L^2 \leq \frac{1}{16}$. 
Thus, the coefficient in front of the second term
in the right hand side is smaller than $\frac{1}{16}$, and rearranging this term we obtain
\begin{align*}
\sum_{j = r(t)}^{t} \norm{\xx_{r(t)} - \xx_j}^2 \leq 
\frac{32 \gamma^2}{5}  \sum_{j = r(t)}^{t}\phi_j(\xx_{r(t)}) 
+ \frac{32 \gamma^2 \tau^2}{15}  \sum_{j = r(t)}^{t -1}\norm{\nabla f(\xx_j)}^2.
\end{align*}
We therefore conclude 
(using the trivial upper bounds $\frac{6 \cdot 32}{5} \leq 48$ and $\frac{6 \cdot 32}{15} \leq 16$) that
$$
\ba{cl}
& \displaystyle \norm{\sum\limits_{j = r(t)}^{t}\left( \nabla f(\xx_j) - \nabla f_{i_j} (\xx_j) \right) }^2 \\
\\
& \displaystyle \leq \;\; 3 \phi_t(\xx_{r(t)}) +  48 \gamma^2 L^2 \tau \sum\limits_{j = r(t)}^{t}\phi_j(\xx_{r(t)}) 
+ 16 \gamma^2 \tau^3 L^2  \sum\limits_{j = r(t)}^{t}\norm{\nabla f(\xx_j)}^2. \qedhere
\ea
$$
\end{proof}

\subsection{Proof of Theorem~\ref{thm:main}}
We analyze separately the iterations $t$ for which the restarts do not happen: $(t + 1) \mymod \tau \neq 0$, and the restart iterations: $(t + 1) \mymod \tau = 0$. 

\paragraph{Iterations without restarts.}
Using $L$-smoothness of $f$, that is implied by Assumption~\ref{as:smooth},
\begin{eqnarray*}
f(\tilde{\xx}_{t  +1}) &\leq& f(\tilde \xx_t) - \gamma \langle \nabla f(\tilde \xx_t), \nabla f(\xx_t) \rangle 
+ \frac{L \gamma^2}{2} \norm{\nabla f(\xx_t)}^2\\
&=& 
f(\tilde \xx_t) - \frac{\gamma}{2} \norm{\nabla f(\tilde \xx_t)}^2 -  \frac{\gamma}{2} \norm{\nabla f(\xx_t)}^2 
+ \frac{\gamma}{2}\norm{\nabla f(\tilde \xx_t) - \nabla f(\xx_t) }^2 \\
&& + \; \frac{L \gamma^2}{2} \norm{\nabla f(\xx_t)}^2 \\
&\leq& f(\tilde \xx_t) - \frac{\gamma}{2} \norm{\nabla f(\tilde \xx_t)}^2 -  \frac{\gamma}{2} \norm{\nabla f(\xx_t)}^2 
+ \frac{\gamma L^2}{2}\norm{\tilde \xx_t - \xx_t}^2 + \frac{L \gamma^2}{2} \norm{\nabla f(\xx_t)}^2.
\end{eqnarray*}

Employing bound \eqref{XtXbartBound} from Lemma~\ref{lem:consensus} 
for the distance between the real and virtual sequences\footnote{For convenience, we set everywhere that $\sum\limits_{j = a}^{b} \ldots \equiv 0$, wherever $a > b$.}:
\beq \label{BoundForIters}
\ba{rcl}
\displaystyle
\norm{\tilde{\xx}_t - \xx_t}^2 & = & 
\displaystyle
\gamma^2\norm{\sum\limits_{j = r(t)}^{t - 1}\left( \nabla f(\xx_j) - \nabla f_{i_j} (\xx_j) \right) }^2,
\ea
\eeq
we get (note that we have a shifted index $t \mapsto t - 1$ in \eqref{XtXbartBound},
and to cover formally the trivial case $t = r(t)$ we use $\bar{\phi}_t(\xx_t)$ in the bound instead of $\phi_{t - 1}(\xx_t)$):
\begin{align}
f(\tilde{\xx}_{t+1}) &\leq f(\tilde \xx_t) - \frac{\gamma}{2} \norm{\nabla f(\tilde \xx_t)}^2 -  \frac{\gamma}{2} \norm{\nabla f(\xx_t)}^2 
+ \frac{3L^2 \gamma^3}{2}  \bar{\phi}_{t}(\xx_{r(t)}) +  \frac{L^2 \gamma^3}{8 \tau}  \sum_{j = r(t)}^{t - 1}\phi_j(\xx_{r(t)})\nonumber\\
&\qquad\qquad + \frac{\tau L^2 \gamma^3}{24} \sum_{j = r(t)}^{t-1}\norm{\nabla f(\xx_j)}^2 + \frac{L \gamma^2}{2}\norm{\nabla f(\xx_t)}^2. \label{eq:descent_no_restart}
\end{align}

\paragraph{Restart Iterations.} Next, we analyse the iterations $t$ for which restarts happens in virtual sequence, i.e.\ $(t + 1) \mymod \tau = 0$. First, we re-write the update as 
\begin{align*}
\tilde \xx_{t + 1} & \stackrel{\eqref{eq:restart}}{=} \xx_{t + 1} \stackrel{\eqref{eq:algo}}{=} \xx_{t } - \gamma \nabla f_{i_t}(\xx_t)  =  \tilde \xx_t +  (\xx_{t } - \tilde \xx_t) - \gamma \nabla f_{i_t}(\xx_t) \\
&= \tilde \xx_t - \gamma \nabla f(\xx_t)  + \gamma \sum_{j = r(t)}^{t}\left( \nabla f(\xx_j) - \nabla f_{i_j}(\xx_j)\right), 
\end{align*}
where we used that $\xx_t - \tilde \xx_t = \gamma \sum\limits_{j = r(t)}^{t - 1}\left( \nabla f(\xx_j) - \nabla f_{i_j}(\xx_j)\right) $ due to the definitions (see \eqref{VirtualRealSeq}). 

We use $L$-smoothness of $f$ that follows from Assumption~\ref{as:smooth},
thus
\begin{align*}
 f(\tilde \xx_{t + 1}) &\leq f(\tilde \xx_t) - \gamma \langle \nabla f(\tilde \xx_t), \nabla f(\xx_t) - \sum_{j = r(t)}^{t}\left( \nabla f(\xx_j) - \nabla f_{i_j}(\xx_j)\right)\rangle \\
 &\qquad \qquad + \frac{L}{2} \gamma^2 \norm{\nabla f(\xx_t) - \sum_{j = r(t)}^{t}\left( \nabla f(\xx_j) - \nabla f_{i_j}(\xx_j)\right)}^2\\
 & \stackrel{\eqref{eq:sum_of_n_vectors}}{\leq} f(\tilde \xx_t) +  \underbrace{(-\gamma \langle \nabla f(\tilde \xx_t), \nabla f(\xx_t) \rangle) }_{:=T_1} +  \underbrace{ \gamma \langle \nabla f(\tilde \xx_t),  \sum_{j = r(t)}^{t}\left(\nabla f(\xx_j) - \nabla f_{i_j}(\xx_j)\right)\rangle}_{:=T_2} \\
 & \qquad\qquad + L \gamma^2 \norm{\nabla f(\xx_t)}^2 + \underbrace{ L \gamma^2 \norm{\sum_{j = r(t)}^{t}\left( \nabla f(\xx_j) - \nabla f_{i_j}(\xx_j)\right)}^2}_{:=T_3 }.
\end{align*}
We further separately estimate terms $T_1$, $T_2$ and $T_3$. We have
$$
\ba{rcl}
\displaystyle
T_1 &=& 
\displaystyle
\frac{\gamma}{2} \norm{\nabla f(\tilde \xx_t) - \nabla f(\xx_t)}^2 - \frac{\gamma}{2} \norm{\nabla f(\tilde \xx_t)}^2 - \frac{\gamma}{2} \norm{\nabla f(\xx_t)}^2\\
\\
&\stackrel{\eqref{eq:l-smooth}}{\leq}& 
\displaystyle
\frac{L^2 \gamma}{2} \norm{\tilde \xx_t - \xx_t}^2 - \frac{\gamma}{2} \norm{\nabla f(\tilde \xx_t)}^2 - \frac{\gamma}{2} \norm{\nabla f(\xx_t)}^2 \\
\\
&\overset{\eqref{BoundForIters}}{=} &
\displaystyle
\frac{L^2 \gamma^3}{2} \norm{ \sum\limits_{j = r(t)}^{t - 1} \bigl( \nabla f(\xx_j) - \nabla f_{i_j}(\xx_j  \bigr) }^2 
- \frac{\gamma}{2} \norm{\nabla f(\tilde \xx_t)}^2 - \frac{\gamma}{2} \norm{\nabla f(\xx_t)}^2
\ea
$$
Using bound \eqref{XtXbartBound} for the first term (with a shifted index $t \mapsto t - 1$,
thus using $\bar{\phi}_t(\xx_t)$ to cover formally the trivial case $t = r(t)$), we obtain 
\begin{align*}
T_1 & \;\;\leq \;\; \frac{3L^2 \gamma^3}{2}  \bar{\phi}_{t}(\xx_{r(t)}) +  \frac{L^2\gamma^3}{8\tau}  \sum_{j = r(t)}^{t - 1}\phi_j(\xx_{r(t)}) 
+\frac{\tau L^2 \gamma^3 }{24} \sum_{j = r(t)}^{t - 1}\norm{\nabla f(\xx_j)}^2 \\
&\qquad - \frac{\gamma}{2} \norm{\nabla f(\tilde \xx_t)}^2 - \frac{\gamma}{2} \norm{\nabla f(\xx_t)}^2.
\end{align*}
The second term $T_2$ can be bounded as follows:
\begin{align*}
T_2 &=  \langle \nabla f(\tilde \xx_t),  
\gamma \sum_{j = r(t)}^{t}\left( \nabla f(\xx_j) - \nabla f_{\pi_j}(\xx_j)\right)\rangle \\
&\stackrel{\eqref{eq:scalar_product_ab} \; \text{with} \; \alpha := 32 \sqrt{3} L}{\leq} \underbrace{\frac{1}{64 \sqrt{3} L} \norm{ \nabla f(\tilde \xx_t)}^2}_{:=T_4} + \underbrace{16\sqrt{3} L\gamma^2 \norm{\sum_{j = r(t)}^{t}\left( \nabla f(\xx_j) - \nabla f_{\pi_j}(\xx_j)\right)}^2}_{:=T_5}.
\end{align*}
We estimate $T_4$ as \footnote{On the second line we could use the equation $\norm{\aa + \bb}^2 \leq (1 + \alpha) \norm{\aa}^2 + (1 + \alpha^{-1}) \norm{\bb}^2$ for some constant $\alpha < 1$ to get a better dependence on the constants.}
\begin{align*}
64 \sqrt{3} T_4 = \frac{1}{L}\norm{\nabla f(\tilde \xx_t)}^2 &= \frac{1}{L \tau} \sum_{j = 0}^{\tau - 1} \norm{\nabla f(\tilde \xx_t)}^2 \\
& \stackrel{\eqref{eq:sum_of_n_vectors}}{\leq} \frac{2}{L \tau} \sum_{j = 0}^{\tau - 1} \norm{\nabla f(\tilde \xx_t) - \nabla f(\tilde \xx_{t - j})}^2 + \frac{2}{L \tau} \sum_{j = 0}^{\tau - 1} \norm{\nabla f(\tilde \xx_{t - j})}^2\\
&\stackrel{\eqref{eq:l-smooth}}{\leq} \frac{2}{L \tau} \sum_{j = 0}^{\tau - 1} L^2 \norm{\tilde \xx_t - \tilde \xx_{t - j}}^2 + \frac{2}{L \tau} \sum_{j = 0}^{\tau - 1}\norm{\nabla f(\tilde \xx_{t - j})}^2\\
&\leq \frac{2\gamma^2 L}{\tau}  \sum_{j = 0}^{\tau - 1} \norm{\sum_{k = t - j}^{t - 1} \nabla f(\xx_k) }^2 + \frac{2}{L \tau} \sum_{j = 0}^{\tau - 1}\norm{\nabla f(\tilde \xx_{t - j})}^2\\
&\stackrel{\eqref{eq:sum_of_n_vectors}}{\leq} 2 \gamma^2 L \sum_{j = 0}^{\tau - 1} \sum_{k = t - j}^{t - 1} \norm{\nabla f(\xx_k) }^2 + \frac{2}{L \tau} \sum_{j = 0}^{\tau - 1} \norm{\nabla f(\tilde \xx_{t - j})}^2\\
&\leq 2 \gamma^2 L \tau \sum_{j = 0}^{\tau - 1} \norm{\nabla f(\xx_{t - j}) }^2 + \frac{2}{L \tau} \sum_{j = 0}^{\tau - 1} \norm{\nabla f(\tilde \xx_{t - j})}^2.
\end{align*}
Further using that $\tau = \left\lfloor\frac{1}{8\sqrt{3} L \gamma}\right\rfloor \geq 1$, which means both that $\tau \leq \frac{1}{8\sqrt{3} L \gamma}$ 
and\footnote{We note that for simplicity of presentation this lower bound is rough and could be refined if we want to improve dependence on the numerical constants.} $\tau \geq \frac{1}{16\sqrt{3} L \gamma}$ we get
\begin{align*}
\frac{1}{L}\norm{\nabla f(\tilde \xx_t)}^2 \leq \frac{\gamma}{4\sqrt{3}} \sum_{j = 0}^{\tau - 1} \norm{\nabla f(\xx_{t - j}) }^2 + 32\sqrt{3} \gamma \sum_{j = 0}^{\tau - 1} \norm{\nabla f(\tilde \xx_{t - j})}^2.
\end{align*}
Thus, 
\begin{align*}
T_4 = \frac{1}{64 \sqrt{3} L}\norm{\nabla f(\tilde \xx_t)}^2 \leq \frac{\gamma}{64 \cdot 12} \sum_{j = 0}^{\tau - 1} \norm{\nabla f(\xx_{t - j}) }^2 + \frac{\gamma}{2}\sum_{j = 0}^{\tau - 1} \norm{\nabla f(\tilde \xx_{t - j})}^2.
\end{align*}

It is left to estimate the terms $T_3 + T_5$, as follows:
$$
\ba{rcl}
T_3 + T_5 & \leq & 
\displaystyle
29 L \gamma^2\norm{\sum\limits_{j = r(t)}^{t}\left( \nabla f(\xx_j) - \nabla f_{\pi_j}(\xx_j)\right)}^2 \\
\\
& \overset{\eqref{XtXbartBound}}{\leq} &
\displaystyle
87 L \gamma^2 \phi_t(\xx_{r(t)})
+ \frac{29 L \gamma^2}{4 \tau} \sum\limits_{j = r(t)}^t \phi_j(\xx_{r(t)})
+ \frac{29\tau L \gamma^2}{12} \sum\limits_{j = r(t)}^t \norm{\nabla f(\xx_j)}^2 \\
\\
& \overset{\tau \leq \frac{1}{8\sqrt{3}L \gamma}}{\leq} &
\displaystyle
87 L \gamma^2 \phi_t(\xx_{r(t)})
+ \frac{29 L \gamma^2}{4 \tau} \sum\limits_{j = r(t)}^t \phi_j(\xx_{r(t)})
+ \frac{\gamma}{6} \sum\limits_{j = r(t)}^t \norm{\nabla f(\xx_j)}^2.
\ea
$$

\paragraph{Summing up estimations for { $T_1$}, $T_2$, $T_3$, $T_4$ and $T_5$.}
After summing up, the descent equation for the restart iterations is

\begin{align}
 & f(\tilde \xx_{t + 1})\quad  \leq 
 \quad f(\tilde \xx_t)
 \quad   {  \underbrace{ 
 \ba{l} \displaystyle
 \quad  - \; {\frac{\gamma}{2}} \norm{\nabla f(\tilde \xx_t)}^2 - {\frac{\gamma}{2}} \norm{\nabla f(\xx_t)}^2 
    + { \frac{3L^2 \gamma^3}{2}}  \bar{\phi}_{t}(\xx_{r(t)})\qquad 
  \\[10pt]
  \displaystyle
  \quad  + \; {\frac{L^2\gamma^3}{8 \tau}  } \sum\limits_{j = r(t)}^{t - 1}\phi_j(\xx_{r(t)})
 +  {\frac{\tau L^2 \gamma^3 }{24}}  \sum\limits_{j = r(t)}^{t -1}\norm{\nabla f(\xx_j)}^2\qquad 
 \ea
  }_{\text{bound for } T_1} } \nonumber\\[5pt]
 & \qquad  + \;\; L \gamma^2 \norm{\nabla f(\xx_t)}^2 
 + \underbrace{  87 L \gamma^2 \phi_t(\xx_{r(t)}) +  \frac{29L \gamma^2 }{4 \tau}  \sum\limits_{j = r(t)}^{t}\phi_j(\xx_{r(t)})  
 + \frac{\gamma}{6} \sum\limits_{i = r(t)}^{t}\norm{\nabla f(\xx_i)}^2}_{\text{bound for } T_3 + T_5}\nonumber\\[5pt]
 & \qquad  + \;\; \underbrace{ \frac{\gamma}{64 \cdot 12} \sum\limits_{i = 0}^{\tau - 1} \norm{\nabla f(\xx_{t - i}) }^2 + \frac{\gamma}{2}\sum_{i = 0}^{\tau - 1} \norm{\nabla f(\tilde \xx_{t - i})}^2.}_{\text{bound for } T_4} \label{eq:descent_restart}
\end{align}

\paragraph{Summing up the descent equations \eqref{eq:descent_no_restart} and \eqref{eq:descent_restart} from all the iterations $t$. }
We will use that
\begin{align*}
    \sum_{t = 0}^T \frac{\tau L^2 \gamma^3 }{24} \sum_{i = r(t)}^{t -1}\norm{\nabla f(\xx_i)}^2 
    \;\; \leq \;\; 
    \frac{\tau^2 L^2 \gamma^3 }{24} \sum_{t = 0}^T  \norm{\nabla f(\xx_t)}^2
\end{align*}
Summing up \eqref{eq:descent_no_restart} and \eqref{eq:descent_restart} for $0 \leq t \leq T$, we get
\begin{align}\label{eq:summed}
 f(\Tilde{\xx}_{T +1}) - f(\tilde \xx_0) 
 &\leq - \frac{\gamma}{2}\sum_{t = 0}^T \norm{\nabla f(\tilde \xx_t)}^2 - \left(\frac{\gamma}{2} - L \gamma^2 \right)\sum_{t = 0}^T \norm{\nabla f(\xx_t)}^2 
 + \frac{\tau^2 \gamma^3 L^2}{24}\sum_{t = 0}^{T} \norm{\nabla f(\xx_t)}^2 \nonumber \\
 &\quad + \frac{3L^2 \gamma^3 }{2}  \sum_{t = 0}^T \bar{\phi}_{t}(\xx_{r(t)}) 
 + \frac{L^2 \gamma^3}{8\tau}  \sum_{t = 0}^T  \sum_{j = r(t)}^{t - 1}\phi_j(\xx_{r(t)}) \nonumber \\
 &\quad + \left(\frac{\gamma}{6} + \frac{\gamma}{64 \cdot 24}\right) \sum_{t = 0}^T \norm{\nabla f(\xx_t)}^2 + \frac{\gamma}{2} \sum_{t = 0}^T \norm{\nabla f(\tilde \xx_{t})}^2 \nonumber \\
 &\quad + \sum_{k = 1}^{\lceil \frac{T}{\tau} \rceil} \left[  87  L \gamma^2 \phi_{k \tau - 1}(\xx_{(k - 1) \tau}) 
 +  \frac{29L\gamma^2 }{4\tau}  \sum_{j = (k - 1)\tau}^{k \tau - 1}\phi_j(\xx_{(k - 1)\tau}) \right],
\end{align}
where on the first two lines are the terms that appears in both of the cases, and on the last two lines are the terms specific to the restart iteration cases, and thus happen once every $\tau$ iterations only. 

It is left to: (i) sum up the coefficients in front of $\sum_{t = 0}^T\norm{\nabla f(\xx_t)}^2$ 
and $\sum_{t = 0}^T\norm{\nabla f(\tilde\xx_t)}^2$ 
(where the latter coeficient is equal to zero), (ii) take the expectation over the choosing orders $i_t$, and estimate terms with $\phi_t(\xx)$ through \cref{def:main_variance}.

First (i) we sup up the coefficients in front of $\sum_{t = 0}^T\norm{\nabla f(\xx_t)}^2$, that is
$$
\ba{rcl}
- \frac{\gamma}{2} + L \gamma^2 + \frac{\tau^2 \gamma^3 L^2}{24}
+ \frac{\gamma}{6} + \frac{\gamma}{64 \cdot 24}
& \overset{\tau \leq \frac{1}{8 \sqrt{3} L \gamma}}{\leq} &
- \frac{\gamma}{2} + L \gamma^2 + \frac{\gamma}{64 \cdot 24 \cdot 3}
+ \frac{\gamma}{6} + \frac{\gamma}{64 \cdot 24} \\
\\
& \overset{\gamma \leq \frac{1}{8\sqrt{3} L}}{\leq} &
- \frac{\gamma}{2} + \frac{\gamma}{8 \sqrt{3}} + \frac{\gamma}{64 \cdot 24 \cdot 3}
+ \frac{\gamma}{6} + \frac{\gamma}{64 \cdot 24} \\
\\
& \leq & -\frac{\gamma}{4}.
\ea
$$

Second (ii), we note that
for any $0 \leq t \leq T$ and $r(t) \leq j \leq  t$,
it holds that  $\E \phi_j (\xx_{r(t)}) \leq\sigma_{\tau}^2$ for $k = \lfloor \frac{t}{\tau} \rfloor$ (see \eqref{ExpPhiBound}).
Thus, we can bound the second line of \eqref{eq:summed} as%

\begin{align*}
    \frac{3L^2 \gamma^3}{2}  \sum_{t = 0}^T \E \bar{\phi}_{t}(\xx_{r(t)}) 
    +  \frac{L^2 \gamma^3}{8\tau}  \sum_{t = 0}^T  \sum_{j = r(t)}^{t - 1} \E \phi_j(\xx_{r(t)}) 
    &\leq \frac{3L^2 \gamma^3}{2}  \sum_{t = 0}^T \sigma_{\tau}^2 
        + \frac{L^2 \gamma^3}{8}  \sum_{t = 0}^T \sigma_{\tau}^2 \\
    &\leq 2 L^2 \gamma^3 \sum_{t = 0}^T \sigma_{\tau}^2\\
    & \leq 2 L^2 \gamma^3 (T + 1) \sigmat.
\end{align*}

Similarly, for the last line of \eqref{eq:summed} we can estimate
\begin{align*}
    & \sum_{k = 1}^{\lceil \frac{T}{\tau} \rceil} \left[  87 L \gamma^2  \E \phi_{k \tau - 1}(\xx_{(k - 1) \tau}) 
    +  \frac{29L\gamma^2 }{4\tau} \sum_{j = (k - 1)\tau}^{k \tau - 1} \E \phi_j(\xx_{(k - 1)\tau}) \right]  \\[10pt]
    &\quad \leq \quad 95 L \gamma^2  \sum_{k = 0}^{\lfloor\frac{T}{\tau} \rfloor} \sigmat 
     \quad \stackrel{\tau \geq \frac{1}{16 \sqrt{3} L \gamma}}{\le} \quad
     95 \cdot 16 \cdot \sqrt{3}  \cdot L^2 \gamma^3 \tau \sum_{k = 0}^{\lfloor\frac{T}{\tau} \rfloor} \sigmat \leq 95 \cdot 16 \cdot \sqrt{3}  \cdot L^2 \gamma^3 (T + 1)\sigmat .
\end{align*}

Putting back these calculations into \eqref{eq:summed}, we get
\begin{align*}
 f(\Tilde{\xx}_{T +1}) - f(\tilde \xx_0) &\leq - \frac{\gamma}{4}\sum_{t = 0}^T \norm{\nabla f(\xx_t)}^2 + A  L^2 \gamma^3 (T + 1)\sigmat,
\end{align*}
where $A = 2633$ is a numerical constant. 

Rearranging, dividing by $T + 1$, and using that $\tilde \xx_0 = \xx_0$, and that $f(\tilde \xx_{T + 1}) \geq f^\star$ we get
\begin{align*}
    \frac{1}{T + 1} \sum_{t = 0}^T \norm{\nabla f(\xx_t)}^2 \leq \cO\biggl( \frac{f(\xx_0) - f^\star}{\gamma T} + L^2 \gamma^2 \frac{\tau}{T} \sigmat \biggr).
\end{align*}

\section{Proofs of the bounds in Table~\ref{tab:sigmas}}
In this section we give upper bounds on $\sigmat$ for the special cases given in \cref{tab:sigmas}. 

\subsection{SGD, \cref{ex:sgd}} 
Recall that according to Def.~\ref{def:main_variance},
\begin{align*}
    \sigmat = \sup_{\xx\in\R^d}\max_{\substack{k = 0, \dots, \lfloor \frac{T}{\tau}\rfloor \\ j=0, \dots, \tau-1}} \E~ \left[ \Big\|\textstyle\sum_{t = k \tau}^{\min\{k \tau + j, T\}} \left(\nabla f_{i_t}(\xx) - \nabla f(\xx)\right) \Big\|^2 \Big|~i_0, \dots, i_{k \tau - 1} \right]. 
\end{align*}
Since in SGD each of $i_t$ are sampled independently uniformly at random from $[n]$, the conditional expectation is equal to unconditional, and we also know that $\E \nabla f_{i_t}(\xx) = \nabla f(\xx)$, and 
\begin{align*}
    &\sup_{\xx \in \R^d} \ \max_{\substack{k = 0, \dots, \lfloor \frac{T}{\tau}\rfloor \\ j=0, \dots, \tau-1}} \E~ \left\|\sum_{t = k \tau}^{\min\{k \tau + j, T\}} \left(\nabla f_{i_t}(\xx) - \nabla f(\xx)\right) \right\|^2 \\
    &\quad \;=\; \quad \sup_{\xx_0 \in \R^d} \ \max_{\substack{k = 0, \dots, \lfloor \frac{T}{\tau}\rfloor \\ j=0, \dots, \tau-1}} \E~ \sum_{t = k \tau}^{\min\{k \tau + j, T\}}  \left\|\nabla f_{i_t}(\xx) - \nabla f(\xx) \right\|^2 \\
    &\quad \stackrel{j = \tau}{\leq} \quad \sup_{\xx_0 \in \R^d}~\max_{k = 0, \dots, \lfloor \frac{T}{\tau}\rfloor} \E \sum_{t = k \tau}^{\min\{k \tau + \tau, T\}} \left\|\nabla f_{i_t}(\xx) - \nabla f(\xx) \right\|^2 \\
    &\quad \;\stackrel{\eqref{eq:sigmasgd}}{\leq}\;\; \max_{k = 0, \dots, \lfloor \frac{T}{\tau}\rfloor} \sum_{t = k \tau}^{\min\{k \tau + \tau, T\}} \sigmasgd \\
    &\quad \;\leq\; \quad \tau \sigmasgd.
\end{align*}
This proves the first bound in Table~\ref{tab:sigmas} that $\sigmat \leq \tau \sigmasgd$ for the SGD algorithm. 

\subsection{Incremental Gradient and Single Shuffle, \cref{ex:ig}, \cref{ex:SS}} 

First, we note that if the interval between $k \tau$ and $k \tau + j$ contains the full epoch inside it, i.e. $\{k' n, \dots, (k' + 1) n\} \subseteq \{k \tau, \dots, k \tau + j\}$ for an integer $k'$, the gradients from this epoch cancel out with the full gradient. This is because $\sum_{i = 1}^n \nabla f_{\pi_i}(\xx) = \sum_{i = 1}^n \nabla f_{i}(\xx) = n \nabla f(\xx)$ and 
thus $\sum_{i = 1}^n (\nabla f_{\pi_i}(\xx)  - \nabla f(\xx))= 0$.

Thus, w.l.o.g. we can assume that

\begin{align*}
    \sum_{t = k \tau}^{k \tau + j} \left(\nabla f_{i_t}(\xx) - \nabla f(\xx)\right) = \sum_{i = j_1}^{n} \left(\nabla f_{i}(\xx) - \nabla f(\xx)\right) + \sum_{i = 1}^{j_2} \left(\nabla f_{i}(\xx) - \nabla f(\xx)\right).
\end{align*}
Moreover, if $j_1 < j_2$ then the full epoch will cancel with the full gradient $\nabla f(\xx)$ and this sum will be equal only to the intersecting part $\sum_{i = j_1}^{j_2} \left(\nabla f_{i}(\xx) - \nabla f(\xx)\right)$. Thus, we will just assume that 
\begin{align*}
    \sum_{t = k \tau}^{k \tau + j} \left(\nabla f_{i_t}(\xx) - \nabla f(\xx)\right) = \sum_{i \in \cS} \left(\nabla f_{i}(\xx) - \nabla f(\xx)\right),
\end{align*}
where $\cS \subset [n]$, and thus $|\cS| \leq \min\{j, n\} \leq \min\{\tau, n\}$. Therefore,
\begin{align*}
    \norm{\sum_{t = k \tau}^{k \tau + j} \left(\nabla f_{i_t}(\xx) - \nabla f(\xx)\right)}^2 &= \norm{\sum_{i \in \cS} \left(\nabla f_{i}(\xx) - \nabla f(\xx)\right)}^2\\[5pt]
    &\stackrel{\eqref{eq:sum_of_n_vectors}}{\leq} |\cS| \sum_{i \in \cS} \norm{\nabla f_{i}(\xx) - \nabla f(\xx)}^2 \\[5pt]
    &\leq |\cS| n \cdot \frac{1}{n}\sum_{i = 1}^n \norm{\nabla f_{i}(\xx_{r(t)}) - \nabla f(\xx_{r(t)})}^2 \\[5pt]
    &\leq |\cS| n \cdot \sigmasgd\\[5pt]
    &\leq  \min\{\tau, n\} n \cdot \sigmasgd.
\end{align*}
This proves the bound, $\sigmakt \leq \min\{\tau, n\} n \sigmasgd$.

\subsection{Random Shuffle, \cref{ex:RR}}\label{sec:RR_proof}

In order to estimate $\sigmat$, we will first estimate $\E~ \left\|\sum_{t = k \tau}^{k \tau + j} \left(\nabla f_{i_t}(\xx) - \nabla f(\xx)\right) \right\|^2 $, 
for a fixed $0 \leq j \leq \tau - 1$, and then we will take the maximum over it. 
We need to consider only $j$ such that $k \tau + j \leq T$. 

First, we note that if the interval between $k \tau$ and $k \tau + j$ contains the full epoch inside it, i.e. $\{k' n, \dots, (k' + 1) n\} \subseteq \{k \tau, \dots, k \tau + j\}$ for an integer $k'$, the gradients from this epoch cancel out with the full gradient. This is because $\sum_{i = 1}^n \nabla f_{\pi_i}(\xx) = \sum_{i = 1}^n \nabla f_{i}(\xx) = n \nabla f(\xx)$ and 
thus $\sum_{i = 1}^n (\nabla f_{\pi_i}(\xx)  - \nabla f(\xx))= 0$.

Next, we note that the interval between $k \tau$ and $k \tau + j$ might overlap with more than one epochs (one epochs is equal to $n$ iterations), but it can contain only at most two incomplete epochs.
Thus, w.l.o.g., we consider that the interval $k \tau$ and $k \tau + j$ intersects with the two (not full) epochs, and the lengths of overlaps are equal to $j_1$ and $j_2$ correspondingly. And thus,
\begin{align*}
    \sum_{t = k \tau}^{k \tau + j} \left(\nabla f_{i_t}(\xx) - \nabla f(\xx)\right) = \sum_{t = 1}^{j_1} \left(\nabla f_{\pi^1_t}(\xx) - \nabla f(\xx)\right) + \sum_{t = 1}^{j_2} \left(\nabla f_{\pi^2_t}(\xx) - \nabla f(\xx)\right)
\end{align*}
where $\pi^1$ and $\pi^2$ correspond to the two permutations, and $j_1 + j_2 \leq j + 1 \leq \tau$. 
Taking the norm and the expectation over permutations $\pi^1$ and $\pi^2$ (note that indices $j_1$ and $j_2$ are non-randomized fixed parameters that depend only on $j, \tau$ and $n$),

We further use that 
\begin{align*}
    &\norm{\sum_{t = 1}^{j_1} \left(\nabla f_{\pi^1_t}(\xx) - \nabla f(\xx)\right) + \sum_{t = 1}^{j_2} \left(\nabla f_{\pi^2_t}(\xx) - \nabla f(\xx)\right)}^2 \\
    &\leq 2 \norm{\sum_{t = 1}^{j_1} \left(\nabla f_{\pi^1_t}(\xx) - \nabla f(\xx)\right)}^2 + 2 \norm{\sum_{t = 1}^{j_2} \left(\nabla f_{\pi^2_t}(\xx) - \nabla f(\xx)\right)}^2
\end{align*}
Taking the conditional expectation, and estimating the two terms using similar calculations to the previous case of Incremental Gradient and Single Shuffle, we arrive that $\sigmat \leq 4 \min\{\tau, n\} n \cdot \sigmasgd$.

\subsection{Single function, \cref{ex:singlefunc}}
Since $i_t \equiv 1~\forall t$, we have, using that $j \leq \tau$:
\begin{align*}
    \norm{\sum_{t = k \tau}^{k \tau + j} \left(\nabla f_{1}(\xx) - \nabla f(\xx)\right)}^2 = \norm{ j \left(\nabla f_{1}(\xx) - \nabla f(\xx)\right)}^2 = j^2 \norm{\nabla f_{1}(\xx) - \nabla f(\xx)}^2 \leq \tau^2 \sigmaone.
\end{align*}

\section{Unavoidable bias in \eqref{eq:single_func_rate} }
We now show that $\sigmaone$ in \eqref{eq:single_func_rate} is an unavoidable bias. Assume that we have two functions $f_1(\xx) = \frac{1}{2}\norm{\xx - \aa}^2$ and $f_2(\xx) = \frac{1}{2}\norm{\xx + \aa}^2$ for some constant vector $\aa$. Then,
\begin{align*}
    \nabla f_1(\xx) = \xx - \aa && \nabla f(\xx) = \xx
\end{align*}
and thus, $\sigmaone = \norm{\aa}^2$. If we only optimize over the function $f_1$, then Algorithm~\eqref{eq:algo} will converge to the optimum of $f_1$ which is $\xx_1^\star = \aa$. Then the norm of the full gradient $\norm{\nabla f(\xx_1^\star)}^2 = \norm{\aa}^2 = \sigmaone$. Thus, the algorithm can converge only to the neighbourhood of the size $\sigmaone$.

\end{document}